\documentclass{article} 
\usepackage{iclr2027_conference,times}

\usepackage{amsmath,amsfonts,bm}

\def\eqref#1{equation~\ref{#1}}

\def\1{\bm{1}}

\DeclareMathAlphabet{\mathsfit}{\encodingdefault}{\sfdefault}{m}{sl}
\SetMathAlphabet{\mathsfit}{bold}{\encodingdefault}{\sfdefault}{bx}{n}

\usepackage[utf8]{inputenc} 
\usepackage[T1]{fontenc}    
\usepackage{amsfonts}       
\usepackage{nicefrac}       
\usepackage{microtype}      
\usepackage{xcolor}         
\usepackage{placeins}      

\usepackage{hyperref}
\usepackage{url}
\usepackage[ruled,vlined]{algorithm2e}
\usepackage{algorithmic}
\usepackage{booktabs}
\usepackage{graphicx}
\usepackage{hyperref}
\usepackage{geometry}
\usepackage{multirow}

\title{Gradient-Gated DPO: Stabilizing Preference Optimization in Language Models}

\author{
Inoussa Mouiche\thanks{Corresponding author.} \\
WASP Lab, School of Computer Science \\
University of Windsor, Windsor, ON N9B 3P4, Canada \\
\texttt{\href{mailto:mouiche@uwindsor.ca}{mouiche@uwindsor.ca}}
\And
Abderaouf Bahi \\
LIMA Lab, Faculty of Science and Technology,\\
Chadli Bendjedid University, El Tarf 36000, Algeria \\
\texttt{\href{mailto:a.bahi@univ-eltarf.dz}{a.bahi@univ-eltarf.dz}}
}

\iclrfinalcopy
\begin{document}

\maketitle
\lhead{Preprint}

\begin{abstract}
Direct Preference Optimization (DPO) improves relative preference by increasing
the margin between chosen and rejected responses, but this objective does not
specify how probability mass should be redistributed to achieve that margin.
Consequently, preference separation can improve even as probability mass
collapses away from both preferred and unrelated responses, a phenomenon
associated with the squeezing effect. We show that this pathology is localized:
continued negative updates become destructive when rejected responses enter
low-probability valleys, where further suppression induces harmful
redistribution through the coupled softmax geometry. We introduce
\emph{Gradient-Gated Preference Optimization} (Gate-DPO), which uses the
current probability geometry of each rejected response to selectively
attenuate its contribution, limiting excessive rejected suppression while
delaying preference-loss saturation.
Gate-DPO is objective-agnostic and composes with existing preference methods.
Across four architectures, two preference datasets, and multiple objectives, gating consistently reduces squeezing and improves chosen-response likelihood, while exhibiting robust reductions in rejected-response collapse across architectures and objectives. Mass-dynamics analysis further shows that large
preference margins can conceal substantial suppression of unrelated responses,
while gating achieves healthier redistribution.
Together, our results show that preference margin alone is insufficient to
characterize optimization health and that selectively controlling
probability dynamics provides a simple, general mechanism for stabilizing
preference optimization. Code for experiments is available at GitHub\footnote{Hidden to preserve anonymity.}
\end{abstract}


\section{Introduction}

Preference optimization has become a standard approach for aligning language
models with human preferences. Direct Preference Optimization
(DPO)~\citep{rafailov2023direct} simplifies reinforcement learning from human
feedback by directly increasing the relative preference margin between chosen
and rejected responses. However, optimizing this margin does not determine
\emph{how} probability mass is redistributed across the model's output
distribution. A larger preference margin can therefore coexist with decreasing
likelihood of preferred responses and substantial suppression of unrelated
outputs, as observed in recent analyses of squeezing and likelihood
displacement~\citep{ren2025learning,razin2025unintentional}.
This distinction between the \emph{preference objective} and its induced
\emph{probability dynamics} is central to our work. DPO optimizes
\[
    \log \pi_\theta(y^+\mid x)-\log \pi_\theta(y^-\mid x),
\]
through opposing gradient contributions from the chosen and rejected
responses. Negative rejected-response updates are necessary for preference
learning, but their effect depends on the current probability geometry.
When rejected responses have already entered low-probability ``valleys,''
continued suppression provides diminishing preference-learning benefit while
softmax coupling can redistribute probability mass toward high-probability
outputs and away from other plausible responses. This produces the
\emph{squeezing effect}: the preference margin may continue to improve even as
the broader distribution deteriorates.
Existing approaches mitigate aspects of this behavior through data
augmentation~\citep{ren2025learning}, alternative objectives such as
IPO~\citep{azar2023general} and Cal-DPO~\citep{xiao2024caldpo}, or global
control of the rejected term, as in DPO-Shift~\citep{yang2025dposhift}.
These methods can improve chosen-response likelihood or alter the
margin--likelihood trade-off, but they do not explicitly identify which
rejected responses have entered the low-probability regime. Moreover, the
preference margin alone cannot reveal whether a smaller margin reflects
insufficient rejection or healthier preservation of probability elsewhere in
the distribution. This motivates an intervention that is local to the
probability state of each preference pair.

We introduce \emph{Gradient-Gated Preference Optimization} (Gate-DPO), a
probability-aware mechanism that selectively attenuates rejected-response
updates according to their current probability geometry. The gate remains
near standard preference optimization for informative rejected responses and
progressively attenuates further suppression once they enter low-probability
valleys. Because gating acts inside the preference objective, it has a second
effect beyond direct attenuation: by limiting growth of the effective
preference margin, it can delay loss saturation and sustain the
chosen-response learning signal. The gate is detached from the computation
graph and is independent of the underlying preference objective, allowing the
same mechanism to compose naturally with DPO, IPO, Cal-DPO, DPO-Shift, and
data-augmentation approaches.
Our experiments support this mechanism across four architectures, two
preference datasets, and multiple preference objectives. Gating consistently reduces squeezing and improves chosen-response likelihood, with the effect persisting across architectures, datasets, and preference objectives. Probability-mass analysis further shows that preference margin alone
can conceal substantial suppression of unrelated responses: methods can
achieve strong separation through markedly different redistribution dynamics.
Finally, pairwise response evaluation provides complementary behavioral
evidence, with gated variants exceeding consistently demonstrating their superiority across 32 judge--architecture--objective--gate comparisons.
Together, these results suggest that stable preference optimization requires
considering not only \emph{whether} the preference margin increases, but
\emph{how} that margin is achieved.\\
\noindent\textbf{Our main contributions are:} i) We distinguish preference-margin optimization
from the probability dynamics used to achieve it, and show that continued
suppression of already low-probability rejected responses can induce squeezing
and destructive redistribution through softmax coupling; ii) we propose
Gate-DPO, a probability-aware mechanism that selectively attenuates
rejected-response updates in this localized regime, and show that gating acts
through two coupled mechanisms, direct rejected-response attenuation and
delayed preference-loss saturation, which sustains the chosen-response learning
signal; iii) we demonstrate that gradient gating is modular across architectures and preference objectives, remaining effective when composed with IPO, Cal-DPO, DPO-Shift, and extended SFT; and iv) we show
empirically that preference margin alone can obscure optimization health:
comparable or larger margins can arise from substantially different probability
redistribution over preferred and unrelated responses, while pairwise
evaluations provide evidence that healthier dynamics can translate into
improved response preference.


\section{DPO and Probability Dynamics}

Given a preference pair $(y^+,y^-)$ for prompt $x$, Direct Preference
Optimization (DPO)~\citep{rafailov2023direct} minimizes
\begin{equation}
\mathcal{L}_{\mathrm{DPO}}
=
-\log \sigma\!\left(
\beta
\left[
\log\frac{\pi_\theta(y^+\mid x)}{\pi_{\mathrm{ref}}(y^+\mid x)}
-
\log\frac{\pi_\theta(y^-\mid x)}{\pi_{\mathrm{ref}}(y^-\mid x)}
\right]
\right),
\label{eq:dpo}
\end{equation}
where $\pi_\theta$ is the policy, $\pi_{\mathrm{ref}}$ is the reference
policy, and $\beta$ controls the strength of the preference update. Thus, DPO
optimizes the \emph{relative margin} between chosen and rejected responses; it
does not directly constrain their absolute likelihoods or how probability mass
is redistributed over other outputs.

Let the quantity inside $\beta[\cdot]$ in Eq.~\ref{eq:dpo} be
$\Delta_\theta$. Its gradient has the form
\begin{equation}
\nabla_\theta \mathcal{L}_{\mathrm{DPO}}
=
-c_\theta
\left[
\nabla_\theta\log\pi_\theta(y^+\mid x)
-
\nabla_\theta\log\pi_\theta(y^-\mid x)
\right],
\qquad
c_\theta=\beta\sigma(-\beta\Delta_\theta)>0.
\label{eq:dpo_gradient}
\end{equation}
Hence, margin optimization induces opposing chosen- and rejected-response
updates. These updates operate through a softmax distribution, where changing
one output necessarily affects probability assigned elsewhere. Consequently,
two policies can improve the same preference objective through substantially
different probability-redistribution dynamics.

\subsection{The Squeezing Pathology}
\label{sec:squeezing}

Negative updates to rejected responses are fundamental to preference learning
and are not inherently pathological. The problem arises when rejected
responses enter extremely low-probability regions: continued suppression
provides diminishing preference-learning benefit while softmax coupling can
redistribute probability mass destructively. We refer to this behavior as the
\emph{squeezing effect}, following \citet{ren2025learning}.
Figure~\ref{fig:three_issues} illustrates three interrelated manifestations
during DPO training.

\begin{figure*}[t]
\centering
\includegraphics[width=\textwidth]{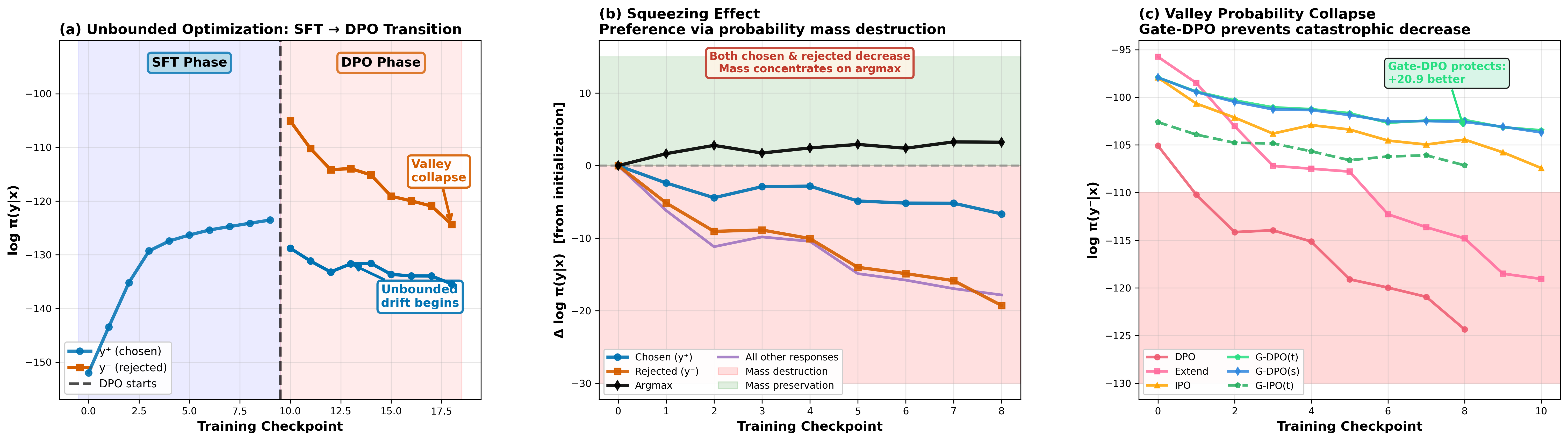}
\caption{Three interrelated pathologies under DPO on Pythia-410M with
Anthropic-HH.
\textbf{(a)} Absolute drift: after the SFT$\rightarrow$DPO transition,
chosen and rejected log-probabilities can both decrease despite improving
preference separation.
\textbf{(b)} Squeezing: probability mass concentrates on high-confidence
outputs while chosen and rejected responses decline.
\textbf{(c)} Valley collapse: already low-probability regions undergo
disproportionately severe suppression.}
\label{fig:three_issues}
\end{figure*}

\noindent\textbf{i) Unbounded absolute drift.}
Because DPO constrains relative preference rather than absolute likelihood,
the preference margin can improve while both chosen and rejected
log-probabilities decrease (Fig.~\ref{fig:three_issues}a).

\noindent\textbf{ii) Squeezing via mass concentration.}
Softmax coupling redistributes probability removed from one region elsewhere
in the distribution. Because the preference objective does not control where
this mass goes, it can concentrate on already high-probability outputs while
chosen, rejected, and unrelated alternatives are jointly suppressed
(Fig.~\ref{fig:three_issues}b).

\noindent\textbf{iii) Valley collapse.}
The rejected-response contribution does not explicitly account for whether
the rejected response already occupies an extremely low-probability region.
Continued negative updates can therefore disproportionately suppress such
``valley'' responses and nearby regions of the distribution
(Fig.~\ref{fig:three_issues}c).

On Pythia-410M with Anthropic-HH, standard DPO decreases chosen and rejected
log-likelihoods by $11.90$ and $26.35$ log units, respectively, while
increasing the argmax by $5.85$ and suppressing the remainder of the
distribution by $24.76$. Appendix~\ref{app:softmax_pathology} shows that the
underlying failure mode is not specific to DPO or language models, reproducing
analogous redistribution under negative updates to low-probability classes in
multiclass logistic regression.
These dynamics also show why preference margin alone can be misleading. A
large margin may arise from strong rejected-response suppression while chosen
and unrelated responses simultaneously lose likelihood. Conversely, a method
that better preserves chosen and unrelated responses may obtain a smaller
margin without assigning substantially more likelihood to rejected responses.
Thus, margin improvement characterizes \emph{preference separation}, but not
the health of the underlying probability redistribution.
This observation suggests that rejected gradients should not be weakened
uniformly. Many rejected responses remain informative; attenuation is needed
specifically when a rejected response enters the low-probability regime. This
motivates a localized, model-state-dependent principle:
\emph{localized probability pathology requires localized gradient
attenuation}. We develop this mechanism next.

\section{Gradient-Gated Preference Optimization}
\label{sec:method}

The analysis above suggests that rejected-response suppression should be
attenuated based on the response's current probability state rather than
uniformly across training examples. Gate-DPO implements this principle using
a detached, probability-aware gate.

\subsection{Detecting Low-Probability Valleys}
\label{sec:valley}

For a rejected response $y^-=(y^-_1,\ldots,y^-_T)$, let
$p_t=\pi_\theta(y^-_t\mid x,y^-_{<t})$ denote its teacher-forced token
probabilities. We consider two statistics that summarize whether the rejected
response has entered a low-probability region. The sequence-level statistic is
the geometric mean
\begin{equation}
s_{\mathrm{seq}}(x,y^-)
=
\exp\left(\frac{1}{T}\sum_{t=1}^{T}\log p_t\right),
\label{eq:seq_stat}
\end{equation}
while the token-level statistic uses a lower-tail quantile,
\begin{equation}
s_{\mathrm{tok}}(x,y^-)
=
Q_q\left(\{p_t\}_{t=1}^{T}\right),
\label{eq:token_stat}
\end{equation}
where $Q_q$ is the $q$-quantile. The former summarizes the response globally,
whereas the latter detects localized low-probability tokens that may be hidden
by sequence averaging. Both statistics produce one scalar gate per rejected
response; q10 is therefore sequence-level attenuation conditioned on a
token-level lower-tail statistic, not per-token gradient gating.

\subsection{Probability-Aware Gradient Gating}
\label{sec:gating}

Given either valley statistic $s(x,y^-)$, we define
\begin{equation}
g(x,y^-)
=
\sigma\!\left(\alpha\,[s(x,y^-)-\tau]\right),
\qquad 0<g(x,y^-)<1,
\label{eq:gate}
\end{equation}
where $\tau$ determines when attenuation begins and $\alpha$ controls the
transition sharpness. The gate approaches one outside low-probability regions
and decreases as the rejected response moves deeper into a probability
valley. Because $s(x,y^-)\ge0$, $g$ has an exact floor
$g_{\min}=\sigma(-\alpha\tau)$, attained as $s\to0$: the intervention is a
bounded soft damper, not a mechanism that can drive the rejected contribution
to zero. This floor is negligible when $\alpha\tau\gg1$ (our sequence-level
default, $\tau=0.10,\alpha=50$, gives $g_{\min}\approx0.7\%$) but not when
$\alpha\tau=O(1)$ (our token-level default, $\tau=0.005,\alpha=50$, gives
$g_{\min}\approx43.8\%$); Appendix~\ref{app:gradient_analysis} makes this
precise and Appendix~\ref{app:tau_sensitivity} reports the corresponding
logged gate statistics.

Crucially, $g$ is detached from the computation graph. It therefore acts as
a model-state-dependent coefficient rather than an additional optimization
target. The gate is bounded and monotone, smoothly interpolating between strong
attenuation in low-probability valleys and the ungated rejected update
elsewhere; Appendix~\ref{app:gradient_analysis} analyzes its properties and
resulting gradient dynamics. For DPO, we replace the rejected log-ratio with its gated counterpart:
\begin{equation}
\mathcal{L}_{\mathrm{Gate\text{-}DPO}}
=
-\log\sigma\!\left(
\beta
\left[
r_\theta^+
-
g(x,y^-)\,r_\theta^-
\right]
\right),
\label{eq:gate_dpo}
\end{equation}
where
\[
r_\theta^\pm
=
\log\frac{\pi_\theta(y^\pm\mid x)}
{\pi_{\mathrm{ref}}(y^\pm\mid x)}.
\]
Gating also changes the common gradient coefficient
\begin{equation}
c_\theta^{\mathrm{gate}}
=\beta\left[1-\sigma\!\left(\beta[r_\theta^+-g r_\theta^-]\right)\right].
\label{eq:gated_gradient_coefficient}
\end{equation}
For $r_\theta^-<0$ and $g<1$, the effective margin is smaller than in DPO at
the same policy state, so $c_\theta^{\mathrm{gate}}>c_\theta^{\mathrm{DPO}}$.
Gate-DPO therefore acts through two coupled mechanisms: direct attenuation of
the rejected term and delayed loss saturation, which sustains the
chosen-response learning signal. Appendix~\ref{app:desaturation} provides
the derivation and empirical analysis.

\subsection{Locality and Modularity}
\label{sec:modularity}

Gate-DPO differs from globally applied interventions in that attenuation is
determined independently for each preference pair from the model's current
probability geometry. Rejected responses in the same minibatch can therefore
receive different attenuation, and their gates evolve as the policy changes.
This directly matches the localized nature of the pathology identified in
Section~\ref{sec:squeezing}.
The mechanism is also independent of the particular preference objective:
whenever an objective contains a rejected-response contribution, the same
detached gate can modulate that contribution without changing how the valley
state is detected. We therefore apply gating not only to DPO, but also to IPO,
Cal-DPO, DPO-Shift, and extended-SFT variants. Section~\ref{sec:experiments}
evaluates this modularity empirically. Appendix~\ref{app:additional_gate_extensions} further shows analytically that the same construction recovers RPO and Bounded-DPO extensions.

\section{Experiments}
\label{sec:experiments}
We evaluate whether Gate-DPO mitigates squeezing, improves probability
redistribution, generalizes across architectures and datasets, and composes
with other preference objectives. Pairwise response evaluation further tests
whether healthier optimization dynamics yield more preferred outputs.

\subsection{Experimental Setup}
\label{sec:exp_setup}

\paragraph{Models and datasets:}
we evaluate four architectures, Pythia-410M~\citep{biderman2023pythia},
Qwen-0.5B~\citep{bai2023qwen}, LLaMA-7B~\citep{touvron2023llama}, and
Qwen1.5-7B~\citep{bai2023qwen} on Anthropic-HH~\citep{bai2022training}
and UltraFeedback~\citep{cui2023ultrafeedback}. This provides two model
families at the 7B scale in addition to the smaller architectures.
\textbf{Training:}
models are initialized from SFT checkpoints and preference-optimized for
5--6 epochs with batch size 4 and learning rate $5\times10^{-7}$. For
Gate-DPO, we use $\tau=0.10$ for sequence-level gating and $\tau=0.005$
for token-level gating, with $q=0.10$ and $\alpha=50$. Full implementation details, computational cost, and hyperparameter
sensitivity analyses are provided in Appendix~\ref{app:implementation}.
\textbf{Evaluation:}
we measure changes in chosen and rejected log-likelihood, preference
separation, and gate activation on held-out data. To directly characterize
probability redistribution, we additionally measure changes over preferred,
rejected, and unrelated response variants relative to the SFT reference.
The full mass-dynamics protocol, metric definitions, and representative
training and evaluation examples are provided in
Appendix~\ref{app:mass_dyn_prot}. Finally, we use pairwise response evaluation
as complementary task-level evidence.

\subsection{Comparison with Existing Mitigations}
\label{sec:main_comparison}

Figure~\ref{fig:method_comparison} compares gradient gating with four
representative approaches on Pythia-410M with Anthropic-HH: extended
SFT~\citep{ren2025learning}, Cal-DPO~\citep{xiao2024caldpo},
IPO~\citep{azar2023general}, and DPO-Shift~\citep{yang2025dposhift}.
\textbf{Ungated baselines:}
standard DPO exhibits severe squeezing, with
$\Delta_{\mathrm{chosen}}=-11.9$.
Extended SFT does not resolve the pathology
($-12.2$), indicating that improving the initialization alone is
insufficient to prevent destructive preference updates.
DPO-Shift similarly remains strongly negative ($-10.6$), showing that a
global schedule on the rejected term does not resolve squeezing in this
setting.
IPO ($-1.0$) substantially reduces the collapse but still leaves
chosen-response likelihood decreasing. Cal-DPO's calibration term alone is
enough to reverse it ($+1.0$), since calibrating the implicit reward toward a
fixed target opposes unbounded margin growth even without gating.
\textbf{Effect of gating:}
Gate-DPO reverses this behavior, yielding
$\Delta_{\mathrm{chosen}}=+1.2$ with sequence-level gating and
$+1.3$ with token-level gating.
The same mechanism also improves other objectives:
Gate-Cal-DPO further increases Cal-DPO's already-positive recovery to
$+1.5$ to $+1.6$, while Gate-DPO-Shift reaches
$+1.2$ to $+1.6$ compared with $-10.6$ for its
ungated baseline. Gate-IPO likewise reverses the residual drift, reaching
$+1.0$ to $+1.2$ with token-level gating, compared with $-1.0$ for
ungated IPO. Thus, the benefit is not specific to the DPO objective.
The training trajectories provide a consistent picture: ungated methods
exhibit decreasing chosen likelihood together with increasing
argmax probability, whereas gated variants substantially suppress this
concentration while stabilizing rejected-response updates.
Token-level gating consistently provides a modest additional improvement
over sequence-level gating across all four objectives, suggesting that
lower-tail token probabilities capture localized valleys somewhat more
precisely than a sequence-level average. Appendix~\ref{app:desaturation} independently confirms delayed saturation
through the coefficient $\bar c_\theta$, with gating slowing its decay in
seven of the eight tested architecture--dataset configurations.
\begin{figure}[t]
\centering
\includegraphics[width=\linewidth,height=2.7in]
{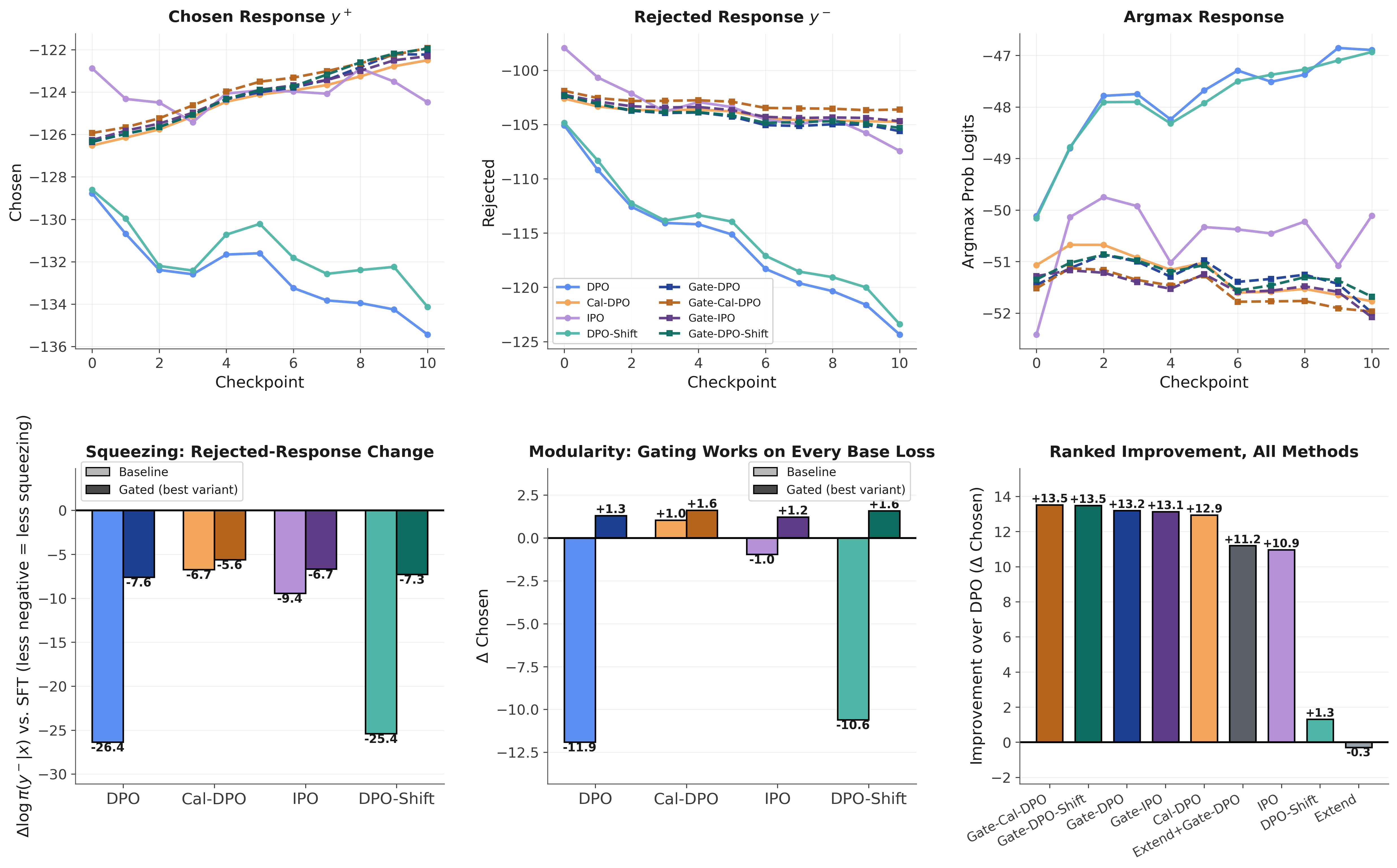}
\caption{\textbf{Comparison with existing mitigation strategies on
Pythia-410M using 5k Anthropic-HH preference pairs.}
Top: learning trajectories for chosen, rejected, and argmax responses.
Bottom: rejected-response squeezing, chosen-response change, and improvement
relative to standard DPO. Gating consistently reduces squeezing and improves
chosen-response dynamics across the evaluated preference objectives.}
\label{fig:method_comparison}
\end{figure}

\subsection{Probability-Mass Dynamics}
\label{sec:mass_dynamics}

Preference margin measures whether chosen responses become more preferred
than rejected responses, but not how the broader output distribution changes
to achieve that separation. We therefore evaluate probability-mass dynamics
over preferred, rejected, and unrelated response variants. We report
$\Delta$ Margin, $\Delta$ Chosen, and $\Delta$ Others relative to the SFT
reference; the complete protocol and metric definitions are provided in
Appendix~\ref{app:mass_dyn_prot}.

Table~\ref{tab:mass_dynamics} reveals markedly different redistribution
dynamics behind similar preference improvements. Standard DPO obtains the
largest margin ($14.45$), but simultaneously decreases preferred responses
($\Delta\text{Chosen}=-11.90$) and strongly suppresses unrelated responses
($\Delta\text{Others}=-27.55$). DPO-Shift exhibits a similar pattern:
despite a margin of $14.81$, chosen and unrelated responses decrease by
$10.60$ and $25.07$ log units, respectively. Thus, a large preference
margin can coexist with substantial probability collapse outside the
chosen-rejected comparison.
Gating changes this trade-off. Gate-DPO achieves smaller but still substantial
margins ($8.48$--$8.90$) while reversing chosen-response drift
($+1.17$ to $+1.30$) and substantially reducing collateral suppression
($-1.42$ to $-3.86$). The same pattern appears when gating is composed with
other methods. In particular, Gate-DPO-Shift converts the negative chosen
drift of DPO-Shift to $+1.23$--$+1.58$ while reducing
$\Delta\text{Others}$ from $-25.07$ to between $-1.65$ and $-3.37$.
Extend + Gate-DPO (q10) similarly restores positive chosen movement
($+1.15$) while limiting collateral suppression ($-1.96$) and does so
robustly across all four architectures (Appendix~\ref{app:modularity_empirical}).
The sequence-level Extend variant uses one untuned $\tau=0.03$ across
architectures and reduces rejected collapse on three of four models
(Table~\ref{tab:modularity_full}); on Pythia-410M its gate barely activates
and chosen movement remains negative ($-5.42$). This is consistent with the
model--dataset sensitivity of $\tau$ examined in Section~\ref{sec:generalization}
and Appendix~\ref{app:tau_sensitivity}; q10 is positive here under its shared
default.
These results also clarify why a smaller preference margin should not
automatically be interpreted as weaker preference learning. The reduction
does not necessarily correspond to undesirable recovery of the rejected
response; it can accompany substantially better preservation of preferred
and unrelated regions of the distribution. Preference margin therefore
characterizes \emph{separation}, whereas the accompanying mass dynamics
characterize \emph{how that separation is achieved}.

\begin{table}[thp]
\centering
\small
\caption{\textbf{Probability-mass dynamics on Pythia-410M with
Anthropic-HH.} $\Delta$ Margin measures the change in chosen--rejected
separation; $\Delta$ Chosen measures movement on preferred responses; and
$\Delta$ Others measures collateral movement on unrelated responses.
Higher $\Delta$ Chosen and $\Delta$ Others closer to zero indicate healthier
redistribution.}
\label{tab:mass_dynamics}
\begin{tabular}{@{}lrrr@{}}
\toprule
\textbf{Method}
& \textbf{$\Delta$ Margin}
& \textbf{$\Delta$ Chosen}
& \textbf{$\Delta$ Others} \\
\midrule
DPO                    & 14.45 & -11.90 & -27.55 \\
Extend + DPO           &  8.86 & -12.20 & -20.27 \\
Cal-DPO                &  7.77 &  +1.03 &  -3.02 \\
IPO                    &  8.48 &  -0.95 &  -7.82 \\
DPO-Shift              & 14.81 & -10.60 & -25.07 \\
\midrule
Gate-DPO (seq)         &  8.48 &  +1.17 &  -1.42 \\
Gate-DPO (q10)         &  8.90 &  +1.30 &  -3.86 \\
Extend + Gate-DPO (seq)& 11.45 &  -5.42 & -11.57 \\
Extend + Gate-DPO (q10)&  8.20 &  +1.15 &  -1.96 \\
Gate-Cal-DPO (seq)     &  7.15 &  +1.51 &  \textbf{-1.04} \\
Gate-Cal-DPO (q10)     &  7.22 &  \textbf{+1.61} &  -2.06 \\
Gate-IPO (seq)         &  7.43 &  +0.96 &  -1.46 \\
Gate-IPO (q10)         &  7.91 &  +1.22 &  -2.76 \\
Gate-DPO-Shift (seq)   &  8.47 &  +1.23 &  -1.65 \\
Gate-DPO-Shift (q10)   &  8.87 &  +1.58 &  -3.37 \\
\bottomrule
\end{tabular}
\end{table}

\subsection{Generalization Across Architectures and Datasets}
\label{sec:generalization}

We next evaluate whether squeezing and the effect of gating generalize across
architectures and preference datasets. We consider four models, Pythia-410M,
Qwen-0.5B, LLaMA-7B, and Qwen1.5-7B on Anthropic-HH and UltraFeedback,
yielding eight model--dataset configurations. Figure~\ref{fig:cross_arch}
summarizes the resulting chosen-response dynamics under DPO and Gate-DPO.
\begin{figure}[t]
\centering
\includegraphics[width=\linewidth, height=2.7in]{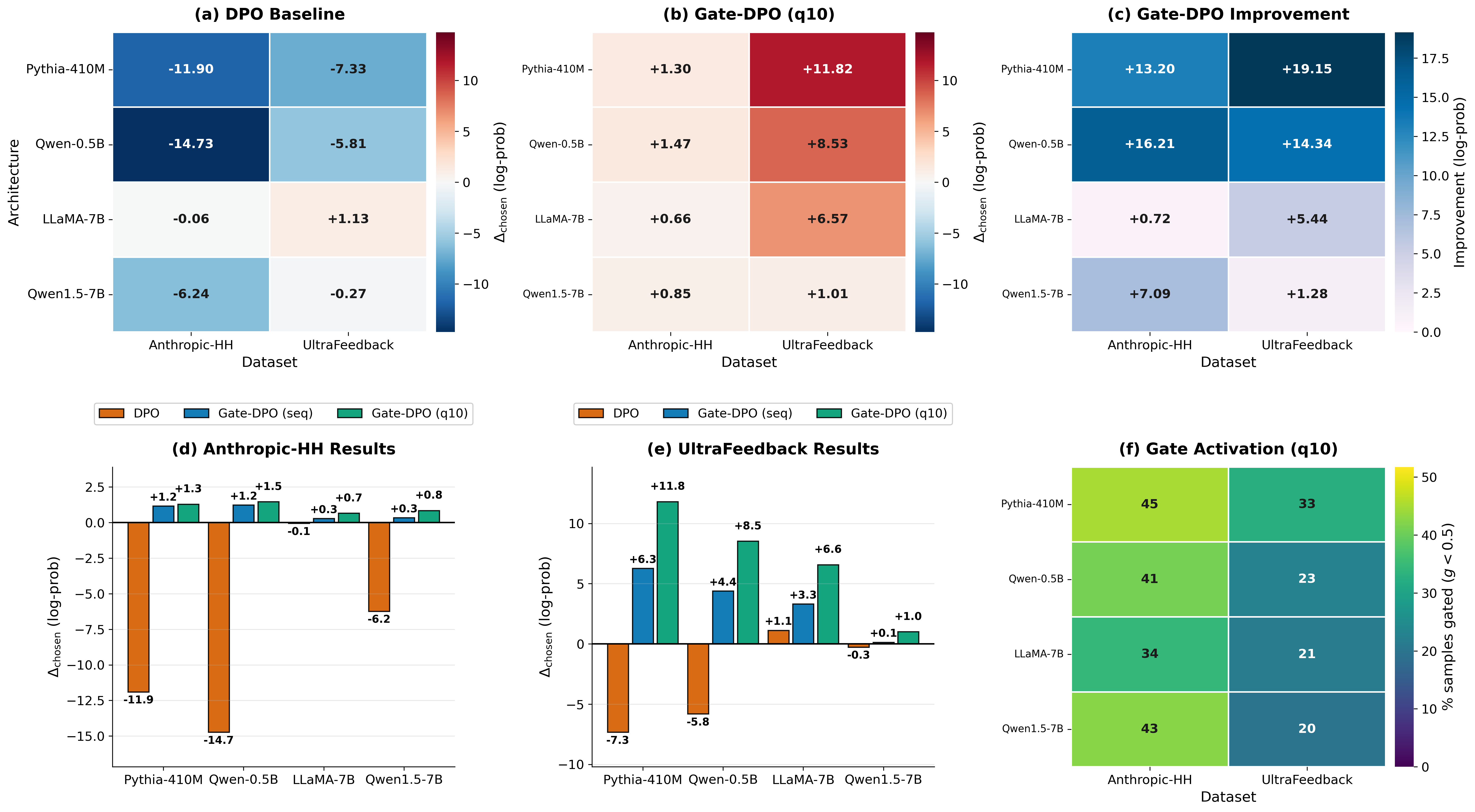}
\caption{\textbf{Generalization across architectures and datasets.}
Chosen-response likelihood change under DPO and Gate-DPO across four
architectures on Anthropic-HH and UltraFeedback. Squeezing varies
substantially across model--dataset combinations, whereas gating produces
positive chosen-response movement in all eight settings using the fixed
default thresholds.}
\label{fig:cross_arch}
\end{figure}\\
\noindent\textbf{Squeezing is model--dataset dependent.}
Standard DPO ranges from severe chosen-response drift on Anthropic-HH
(Pythia-410M: $-11.90$; Qwen1.5-7B: $-6.24$) to almost none on LLaMA-7B
($-0.06$); on UltraFeedback, LLaMA-7B instead moves positively ($+1.13$).
Thus, scale alone does not determine squeezing. Gate-DPO produces positive
chosen-response movement in all eight configurations, with the largest gains
where baseline squeezing is strongest (e.g., Pythia-410M/UltraFeedback:
$-7.33\to+11.82$). The same default thresholds are used throughout despite
gate activation varying by approximately $20$--$45\%$, supporting gating as
a localized intervention. The threshold $\tau$ remains the primary
model--dataset-sensitive hyperparameter; Appendix~\ref{app:tau_sensitivity}
studies its operating range, while Appendix~\ref{app:desaturation} shows how
it regulates loss saturation. Appendix~\ref{app:implementation} shows that
$\alpha$ is comparatively robust.

\subsection{Modularity Across Preference Objectives}
\label{sec:modularity_results}
We test whether gating composes with changes to the objective (IPO and
Cal-DPO), rejected-response schedule (DPO-Shift), and initialization (extended
SFT). Table~\ref{tab:modularity_compact} reports rejected-response dynamics
across four architectures for each base method and its best gated variant.
\begin{table}[t]
\centering
\footnotesize
\setlength{\tabcolsep}{3.5pt}
\caption{\textbf{Modularity across objectives and architectures.}
$\Delta_{\mathrm{rejected}}$: base $\rightarrow$ best gated (minimum
$|\Delta_{\mathrm{rejected}}|$); large negative values indicate stronger
collapse. Full results: Appendix~\ref{app:modularity_empirical}.}
\label{tab:modularity_compact}
\begin{tabular*}{\linewidth}{@{\extracolsep{\fill}}lrrrr@{}}
\toprule
\textbf{Method}
& \textbf{Pythia-410M}
& \textbf{LLaMA-7B}
& \textbf{Qwen-0.5B}
& \textbf{Qwen1.5-7B} \\
\midrule
DPO
& $-26.35\!\rightarrow\!-7.31$
& $-0.88\!\rightarrow\!+0.08$
& $-30.47\!\rightarrow\!-6.25$
& $-10.45\!\rightarrow\!-0.77$ \\
DPO-Shift
& $-25.40\!\rightarrow\!-7.25$
& $-0.80\!\rightarrow\!+0.10$
& $-27.23\!\rightarrow\!-6.00$
& $-9.20\!\rightarrow\!-0.71$ \\
IPO
& $-9.43\!\rightarrow\!-6.46$
& $-0.64\!\rightarrow\!+0.02$
& $-8.36\!\rightarrow\!-5.60$
& $-2.16\!\rightarrow\!-0.75$ \\
Cal-DPO
& $-6.74\!\rightarrow\!-5.61$
& $-0.43\!\rightarrow\!-0.01$
& $-6.03\!\rightarrow\!-4.88$
& $-1.07\!\rightarrow\!-0.52$ \\
Extend
& $-23.26\!\rightarrow\!-7.04$
& $+0.15\!\rightarrow\!+0.27$
& $-22.60\!\rightarrow\!-3.17$
& $-9.35\!\rightarrow\!+0.09$ \\
\bottomrule
\end{tabular*}
\end{table}
Across architectures and objectives, gated variants consistently reduce rejected-response collapse, with the effect persisting across objective modifications, initialization, and DPO-Shift's global schedule. The LLaMA-7B Extend setting is the main exception, where the ungated baseline already exhibits little rejected-response collapse.
The effect persists across objectives, initialization, and DPO-Shift's global
schedule, supporting gating as a composable, sample-specific intervention.
Full results are in Appendix~\ref{app:modularity_empirical}.\\
\noindent\textbf{Pairwise Response Evaluation:}
Table~\ref{tab:pairwise_winrate} provides complementary behavioral evidence
on 30 held-out UltraFeedback prompts using GPT-4 and GPT-5.4 judges. Across
32 judge--architecture--objective--gate comparisons, gated variants exceed
50\%, with win rates ranging from 52.0\% to
68.7\%. This also suggests that healthier optimization dynamics can translate into better task-level performance, consistent with prior work~\citep{ren2025learning,xiao2024caldpo}. Appendix~\ref{app:winrate_protocol} details generation, randomized
response order, judging, and reproducibility.

\begin{table}[t]
\centering
\caption{Pairwise win-rates (\%) of gated variants over their corresponding baselines on UltraFeedback.}
\label{tab:pairwise_winrate}
\setlength{\tabcolsep}{4pt}
\begin{tabular}{llcccccccc}
\toprule
& &
\multicolumn{2}{c}{Gate-DPO} &
\multicolumn{2}{c}{Gate-Cal} &
\multicolumn{2}{c}{Gate-IPO} &
\multicolumn{2}{c}{Gate-Shift} \\
\cmidrule(lr){3-4}
\cmidrule(lr){5-6}
\cmidrule(lr){7-8}
\cmidrule(lr){9-10}
\textbf{Judge} &
\textbf{Arch.} &
seq & q10 &
seq & q10 &
seq & q10 &
seq & q10 \\
\midrule
GPT-4 &
Pythia-410M &
56.7 & 53.3 &
55.0 & 52.0 &
56.0 & 53.7 &
62.1 & 54.0 \\
&
LLaMA-7B &
59.0 & 56.0 &
58.0 & 55.3 &
58.2 & 63.0 &
59.2 & 64.0 \\
\midrule
GPT-5.4 &
Pythia-410M &
60.0 & 61.3 &
56.0 & 52.0 &
60.0 & 56.2 &
59.3 & 60.8 \\
&
LLaMA-7B &
61.1 & 57.0 &
58.0 & 62.7 &
58.6 & 55.3 &
68.7 & 60.0 \\
\bottomrule
\end{tabular}
\end{table}

\FloatBarrier

\section{Related Work}
\label{sec:related_work}

\paragraph{Preference-based alignment.}
Large language models are commonly aligned with human preferences using
reinforcement learning from human feedback (RLHF)
\citep{christiano2017deep,ziegler2019fine,stiennon2020learning,
ouyang2022training,bai2022training}. The standard pipeline combines
supervised fine-tuning, reward modeling from pairwise preferences, and policy
optimization, typically using PPO under a KL constraint
\citep{schulman2017proximal,zhao2023slic}. Although effective, this pipeline
requires multiple model components and careful reinforcement-learning
tuning, motivating direct preference-learning objectives.

\textbf{Direct preference optimization and its variants.}
DPO~\citep{rafailov2023direct} replaces explicit reward modeling and policy
optimization with a closed-form objective derived from the
Bradley--Terry preference model~\citep{bradley1952rank}. Subsequent methods
modify its loss, implicit reward, or token-level credit assignment. These
include IPO~\citep{azar2023general}, which uses a squared objective with a
finite target margin; KTO~\citep{ethayarajh2024kto}, which adopts a
prospect-theoretic formulation; Cal-DPO~\citep{xiao2024caldpo}, which
calibrates implicit rewards; and TIS-DPO~\citep{liu2025tisdpo}, which
estimates token-importance weights using prediction differences from
contrastive models. Yuan et al.~\citep{yuan2024gradient} identify
\emph{gradient entanglement}, showing that coupled preferred- and
rejected-response updates can produce undesirable likelihood dynamics even
while preference accuracy improves.

\textbf{Pathological probability dynamics.}
Recent studies show that an improving preference margin does not necessarily
imply healthier response probabilities. Likelihood displacement refers to
the reduction of preferred-response probability during DPO despite improved
preference separation~\citep{razin2025unintentional}; related asymmetries in
DPO dynamics are analyzed by Feng et al.~\citep{feng2024analyzing}.
Ren et al.~\citep{ren2025learning} identify the \emph{squeezing effect}, in
which negative updates to rejected responses suppress probability across
alternative responses and concentrate mass on already high-confidence
outputs. Their multiclass logistic-regression analysis connects this
behavior to softmax probability dynamics rather than to a particular neural
architecture.\\
Existing mitigation strategies intervene at different levels. Extended
SFT~\citep{ren2025learning} changes model initialization before preference
optimization, whereas DPO-Positive~\citep{pal2024smaug} prevents the
preferred-response likelihood from decreasing. Bounded-DPO
\citep{cho2025rethinking} bounds rejected-response influence while retaining
the basic DPO structure, and DPO-Shift~\citep{yang2025dposhift} introduces a
globally selected coefficient $f(\lambda)$ to control the
chosen--rejected trade-off. ACPO~\citep{huang2026acpo} applies dynamic,
complexity-aware scaling to the rejected reward to counter likelihood
displacement in vision-language alignment. Other work investigates the
effects of data quality, reference policies, and regularization
\citep{yuan2024self,chen2024self,wu2024self,liu2024provably,
wang2024secrets,park2024disentangling}.

\textbf{Relation to Gate-DPO.}
Gate-DPO complements these approaches by making rejected-response suppression
depend on the current probability geometry of each preference pair. Its
detached gate is computed directly from the evolving policy, allowing
examples within the same minibatch to receive different attenuation and
allowing that attenuation to change during training. To our knowledge,
Gate-DPO is the first probability-aware intervention that uses the rejected
response's current policy probability geometry to construct a detached,
sample-specific gate for controlling rejected-response suppression when a
low-probability regime associated with squeezing is detected.\\
This mechanism differs from TIS-DPO, whose probability-derived weights
estimate token importance using auxiliary contrastive models rather than
detecting squeezing from the current rejected-response distribution.
DPO-Shift instead uses a global coefficient that cannot distinguish examples
occupying different probability regimes at the same training step, whereas
Bounded-DPO controls rejected-response influence through a modified objective
rather than a model-state-dependent gate. ACPO also attenuates the rejected
term dynamically, but derives a complexity-aware, target-oriented
coefficient rather than a gate based on the rejected response's current
probability geometry. Gate-DPO therefore targets a distinct failure mode:
continued suppression after a rejected response has entered a
low-probability valley.
Placing the gate inside the preference margin has two coupled effects. It
directly attenuates the rejected-response contribution and regulates
preference-loss saturation, thereby sustaining the chosen-response learning
signal when the rejected response enters the gated regime. The intervention
is also modular: it can be composed with DPO-Shift, IPO, Cal-DPO, and
extended SFT, as examined in our cross-objective experiments.
\paragraph{Relation to generic stabilization mechanisms.}
Gate-DPO differs from generic optimization controls such as $\beta$ tuning,
gradient clipping, and global regularization. Changing $\beta$ rescales the
strength of preference updates across examples, while regularization
constrains the policy globally; neither identifies which rejected responses
have entered a pathological probability regime. Gradient clipping limits
updates according to gradient magnitude, whereas squeezing is determined by
probability geometry and need not coincide with an unusually large gradient
norm. Gate-DPO instead applies sample-specific attenuation using the current
rejected-response probability state while leaving informative rejected
updates largely unchanged. Studying how this local control interacts with
explicit KL constraints~\citep{ziegler2019fine} remains an important
direction for future work.

\section{Conclusion, Limitations, and Future Work}
\label{sec:conclusion}
Preference-margin improvement alone does not characterize healthy optimization:
large margins can coexist with suppression of preferred and unrelated
responses. Gate-DPO applies localized, probability-aware attenuation in
low-probability valleys, reducing rejected-response suppression while delaying
loss saturation and sustaining chosen-response learning. Across architectures,
datasets, and objectives, this improves probability redistribution and reduces
squeezing.

 \textbf{Limitations and Future Work:} Our analysis emphasizes optimization
dynamics, with pairwise evaluation providing initial task-level evidence.
Human judgments, standardized benchmarks (AlpacaEval 2.0 and MT-Bench), and evaluation of the safety--helpfulness trade-off remain important directions for future work. Although fixed thresholds generalize well here, the valley regime varies across models and datasets, suggesting that model--dataset-specific tuning of $\tau$ may yield further gains.


\section*{Generative AI Disclosure}

Generative AI tools were used to assist with writing and editing,
experimental design, analysis, code development and debugging, and literature
review. AI-assisted suggestions, analyses, code, and scientific claims were
reviewed and verified by the authors against experimental results, source
literature, and the implemented methodology. Generative AI was not used as a
substitute for experimental evaluation or as an autonomous source of
scientific evidence. The authors take full responsibility for the methodology, design, proofs,
experiments, results, claims, and final content of this work.



\bibliography{references}

\begin{thebibliography}{31}
\providecommand{\natexlab}[1]{#1}
\providecommand{\url}[1]{\texttt{#1}}
\expandafter\ifx\csname urlstyle\endcsname\relax
  \providecommand{\doi}[1]{doi: #1}\else
  \providecommand{\doi}{doi: \begingroup \urlstyle{rm}\Url}\fi

\bibitem[Azar et~al.(2023)Azar, Rowland, Piot, et~al.]{azar2023general}
Mohammad~Gheshlaghi Azar, Mark Rowland, Bilal Piot, et~al.
\newblock A general theoretical paradigm to understand learning from human
  preferences.
\newblock \emph{arXiv preprint arXiv:2310.12036}, 2023.

\bibitem[Bai et~al.(2023)Bai, Bai, et~al.]{bai2023qwen}
Jinze Bai, Shuai Bai, et~al.
\newblock Qwen technical report.
\newblock \emph{arXiv preprint arXiv:2309.16609}, 2023.

\bibitem[Bai et~al.(2022)]{bai2022training}
Yuntao Bai et~al.
\newblock Training a helpful and harmless assistant with reinforcement learning
  from human feedback.
\newblock \emph{arXiv preprint arXiv:2204.05862}, 2022.

\bibitem[Biderman et~al.(2023)Biderman, Schoelkopf, Anthony, Bradley, O'Brien,
  Hallahan, Khan, Purohit, Prashanth, Raff, et~al.]{biderman2023pythia}
Stella Biderman, Hailey Schoelkopf, Quentin~Gregory Anthony, Herbie Bradley,
  Kyle O'Brien, Eric Hallahan, Mohammad~Aflah Khan, Shivanshu Purohit,
  USVSN~Sai Prashanth, Edward Raff, et~al.
\newblock Pythia: A suite for analyzing large language models across training
  and scaling.
\newblock In \emph{Proceedings of the 40th International Conference on Machine
  Learning (ICML)}, volume 202 of \emph{Proceedings of Machine Learning
  Research}, pp.\  2397--2430. PMLR, 2023.

\bibitem[Bradley \& Terry(1952)Bradley and Terry]{bradley1952rank}
Ralph~Allan Bradley and Milton~E. Terry.
\newblock Rank analysis of incomplete block designs: I. the method of paired
  comparisons.
\newblock \emph{Biometrika}, 39\penalty0 (3/4):\penalty0 324--345, 1952.

\bibitem[Chen et~al.(2024)Chen, Deng, Yuan, Ji, and Gu]{chen2024self}
Zixiang Chen, Yihe Deng, Huizhuo Yuan, Kaixuan Ji, and Quanquan Gu.
\newblock Self-play fine-tuning converts weak language models to strong
  language models.
\newblock \emph{arXiv preprint arXiv:2401.01335}, 2024.

\bibitem[Cho et~al.(2025)Cho, Oh, Kim, and Lee]{cho2025rethinking}
Jae~Hyeon Cho, JunHyeok Oh, Myunsoo Kim, and Byung-Jun Lee.
\newblock Rethinking dpo: The role of rejected responses in preference
  misalignment.
\newblock In \emph{Findings of the Association for Computational Linguistics:
  EMNLP 2025}, 2025.
\newblock arXiv:2506.12725.

\bibitem[Christiano et~al.(2017)Christiano, Leike, Brown, Martic, Legg, and
  Amodei]{christiano2017deep}
Paul~F. Christiano, Jan Leike, Tom Brown, Miljan Martic, Shane Legg, and Dario
  Amodei.
\newblock Deep reinforcement learning from human preferences.
\newblock \emph{Advances in Neural Information Processing Systems}, 30, 2017.

\bibitem[Cui et~al.(2023)Cui, Yuan, Ding, Yao, Zhu, Ni, Xie, Liu, and
  Sun]{cui2023ultrafeedback}
Ganqu Cui, Lifan Yuan, Ning Ding, Guanming Yao, Wei Zhu, Yuan Ni, Guotong Xie,
  Zhiyuan Liu, and Maosong Sun.
\newblock Ultrafeedback: Boosting language models with high-quality feedback.
\newblock \emph{arXiv preprint arXiv:2310.01377}, 2023.

\bibitem[Ethayarajh et~al.(2024)Ethayarajh, Xu, Muennighoff, Jurafsky, and
  Kiela]{ethayarajh2024kto}
Kawin Ethayarajh, Winnie Xu, Niklas Muennighoff, Dan Jurafsky, and Douwe Kiela.
\newblock Kto: Model alignment as prospect theoretic optimization.
\newblock \emph{arXiv preprint arXiv:2402.01306}, 2024.

\bibitem[Feng et~al.(2024)Feng, Qin, Huang, Zhang, and Lei]{feng2024analyzing}
Duanyu Feng, Bowen Qin, Chen Huang, Zheng Zhang, and Wenqiang Lei.
\newblock Towards analyzing and understanding the limitations of dpo: A
  theoretical perspective.
\newblock \emph{arXiv preprint arXiv:2404.04626}, 2024.

\bibitem[Huang et~al.(2026)Huang, Zhang, Shen, Dai, Wang, Deng, and
  Lu]{huang2026acpo}
Kaili Huang, Hongming Zhang, Rui Shen, Linjun Dai, Jiahao Wang, Hanming Deng,
  and Lewei Lu.
\newblock {ACPO}: Counteracting likelihood displacement in vision-language
  alignment with asymmetric constraints.
\newblock \emph{arXiv preprint arXiv:2603.22165}, 2026.
\newblock \doi{10.48550/arXiv.2603.22165}.
\newblock URL \url{https://arxiv.org/abs/2603.22165}.

\bibitem[Liu et~al.(2025)Liu, Bai, Lu, Sun, Kong, Wang, Shan, Jose, Liu, Wen,
  Yu, and Cao]{liu2025tisdpo}
Aiwei Liu, Haoping Bai, Zhiyun Lu, Yanchao Sun, Xiang Kong, Xiaoming~Simon
  Wang, Jiulong Shan, Albin~Madappally Jose, Xiaojiang Liu, Lijie Wen,
  Philip~S. Yu, and Meng Cao.
\newblock {TIS-DPO}: Token-level importance sampling for direct preference
  optimization with estimated weights.
\newblock In \emph{The Thirteenth International Conference on Learning
  Representations}, 2025.
\newblock URL \url{https://openreview.net/forum?id=oF6e2WwxX0}.

\bibitem[Liu et~al.(2024)Liu, Lu, Zhang, Liu, Guo, Yang, Blanchet, and
  Wang]{liu2024provably}
Zhihan Liu, Miao Lu, Shenao Zhang, Boyi Liu, Hongyi Guo, Yingxiang Yang, Jose
  Blanchet, and Zhaoran Wang.
\newblock Provably mitigating overoptimization in rlhf: Your sft loss is
  implicitly an adversarial regularizer.
\newblock In \emph{Advances in Neural Information Processing Systems},
  volume~37, 2024.

\bibitem[Ouyang et~al.(2022)Ouyang, Wu, Jiang, Almeida, Wainwright, Mishkin,
  et~al.]{ouyang2022training}
Long Ouyang, Jeffrey Wu, Xu~Jiang, Diogo Almeida, Carroll Wainwright, Pamela
  Mishkin, et~al.
\newblock Training language models to follow instructions with human feedback.
\newblock \emph{Advances in Neural Information Processing Systems}, 35, 2022.

\bibitem[Pal et~al.(2024)Pal, Karkhanis, Dooley, Roberts, Naidu, and
  White]{pal2024smaug}
Arka Pal, Deep Karkhanis, Samuel Dooley, Manley Roberts, Siddartha Naidu, and
  Colin White.
\newblock Smaug: Fixing failure modes of preference optimisation with
  dpo-positive.
\newblock \emph{arXiv preprint arXiv:2402.13228}, 2024.

\bibitem[Park et~al.(2024)Park, Rafailov, Ermon, and
  Finn]{park2024disentangling}
Ryan Park, Rafael Rafailov, Stefano Ermon, and Chelsea Finn.
\newblock Disentangling length from quality in direct preference optimization.
\newblock \emph{arXiv preprint arXiv:2403.19159}, 2024.

\bibitem[Rafailov et~al.(2023)Rafailov, Sharma, Mitchell, Ermon, and
  Manning]{rafailov2023direct}
Rafael Rafailov, Archit Sharma, Eric Mitchell, Stefano Ermon, and Christopher~D
  Manning.
\newblock Direct preference optimization: Your language model is secretly a
  reward model.
\newblock \emph{arXiv preprint arXiv:2305.18290}, 2023.

\bibitem[Razin et~al.(2025)Razin, Malladi, Bhaskar, Chen, Arora, and
  Hanin]{razin2025unintentional}
Noam Razin, Sadhika Malladi, Adithya Bhaskar, Danqi Chen, Sanjeev Arora, and
  Boris Hanin.
\newblock Unintentional unalignment: Likelihood displacement in direct
  preference optimization.
\newblock In \emph{Proceedings of the International Conference on Learning
  Representations (ICLR)}, 2025.
\newblock arXiv:2410.08847.

\bibitem[Ren \& Sutherland(2025)Ren and Sutherland]{ren2025learning}
Yi~Ren and Danica~J. Sutherland.
\newblock Learning dynamics of llm finetuning.
\newblock In \emph{Proceedings of the International Conference on Learning
  Representations (ICLR)}, 2025.
\newblock arXiv:2407.10490v4.

\bibitem[Schulman et~al.(2017)Schulman, Wolski, Dhariwal, Radford, and
  Klimov]{schulman2017proximal}
John Schulman, Filip Wolski, Prafulla Dhariwal, Alec Radford, and Oleg Klimov.
\newblock Proximal policy optimization algorithms.
\newblock \emph{arXiv preprint arXiv:1707.06347}, 2017.

\bibitem[Stiennon et~al.(2020)Stiennon, Ouyang, Wu, Ziegler, Lowe, Voss,
  Radford, Amodei, and Christiano]{stiennon2020learning}
Nisan Stiennon, Long Ouyang, Jeff Wu, Daniel~M. Ziegler, Ryan Lowe, Chelsea
  Voss, Alec Radford, Dario Amodei, and Paul Christiano.
\newblock Learning to summarize from human feedback.
\newblock In \emph{Advances in Neural Information Processing Systems
  (NeurIPS)}, 2020.

\bibitem[Touvron et~al.(2023)Touvron, Lavril, Izacard, Martinet, Lachaux,
  Lacroix, Rozière, Goyal, Hambro, Azhar, Rodriguez, Joulin, Grave, and
  Lample]{touvron2023llama}
Hugo Touvron, Thibaut Lavril, Gautier Izacard, Xavier Martinet, Marie-Anne
  Lachaux, Timothée Lacroix, Baptiste Rozière, Naman Goyal, Eric Hambro,
  Faisal Azhar, Aurelien Rodriguez, Armand Joulin, Edouard Grave, and Guillaume
  Lample.
\newblock Llama: Open and efficient foundation language models.
\newblock \emph{arXiv preprint arXiv:2302.13971}, 2023.

\bibitem[Wang et~al.(2024)Wang, Dou, Dou, Gao, Hua, Shen, Liu, Jin, Liu, Zhou,
  et~al.]{wang2024secrets}
Binghai Wang, Rui Dou, Shihan Dou, Songyang Gao, Wei Hua, Wei Shen, Yan Liu,
  Senjie Jin, Qin Liu, Yuhao Zhou, et~al.
\newblock Secrets of rlhf in large language models part ii: Reward modeling.
\newblock \emph{arXiv preprint arXiv:2401.06080}, 2024.

\bibitem[Wu et~al.(2024)Wu, Sun, Yuan, Ji, Yang, and Gu]{wu2024self}
Yue Wu, Zhiqing Sun, Huizhuo Yuan, Kaixuan Ji, Yiming Yang, and Quanquan Gu.
\newblock Self-play preference optimization for language model alignment.
\newblock \emph{arXiv preprint arXiv:2405.00675}, 2024.

\bibitem[Xiao et~al.(2024)Xiao, Yuan, Zhu, Li, and Honavar]{xiao2024caldpo}
Teng Xiao, Yige Yuan, Huaisheng Zhu, Mingxiao Li, and Vasant~G. Honavar.
\newblock Cal-dpo: Calibrated direct preference optimization for language model
  alignment.
\newblock \emph{arXiv preprint arXiv:2412.14516}, 2024.

\bibitem[Yang et~al.(2025)Yang, Jiang, Zhang, Zhao, and Li]{yang2025dposhift}
Xiliang Yang, Feng Jiang, Qianen Zhang, Lei Zhao, and Xiao Li.
\newblock Dpo-shift: Shifting the distribution of direct preference
  optimization.
\newblock \emph{arXiv preprint arXiv:2502.07599}, 2025.

\bibitem[Yuan et~al.(2024{\natexlab{a}})Yuan, Zeng, Wu, Wang, Wang, and
  Leqi]{yuan2024gradient}
Hui Yuan, Yifan Zeng, Yue Wu, Huazheng Wang, Mengdi Wang, and Liu Leqi.
\newblock A common pitfall of margin-based language model alignment: Gradient
  entanglement.
\newblock \emph{arXiv preprint arXiv:2410.13828}, 2024{\natexlab{a}}.

\bibitem[Yuan et~al.(2024{\natexlab{b}})Yuan, Pang, Cho, Sukhbaatar, Xu, and
  Weston]{yuan2024self}
Weizhe Yuan, Richard~Yuanzhe Pang, Kyunghyun Cho, Sainbayar Sukhbaatar, Jing
  Xu, and Jason Weston.
\newblock Self-rewarding language models.
\newblock \emph{arXiv preprint arXiv:2401.10020}, 2024{\natexlab{b}}.

\bibitem[Zhao et~al.(2023)Zhao, Joshi, Liu, Khalman, Saleh, and
  Liu]{zhao2023slic}
Yao Zhao, Rishabh Joshi, Tianqi Liu, Misha Khalman, Mohammad Saleh, and
  Peter~J. Liu.
\newblock Slic-hf: Sequence likelihood calibration with human feedback.
\newblock \emph{arXiv preprint arXiv:2305.10425}, 2023.

\bibitem[Ziegler et~al.(2019)Ziegler, Stiennon, Wu, Brown, Radford, Amodei,
  Christiano, and Irving]{ziegler2019fine}
Daniel~M. Ziegler, Nisan Stiennon, Jeffrey Wu, Tom~B. Brown, Alec Radford,
  Dario Amodei, Paul Christiano, and Geoffrey Irving.
\newblock Fine-tuning language models from human preferences.
\newblock \emph{arXiv preprint arXiv:1909.08593}, 2019.

\end{thebibliography}
\bibliographystyle{iclr2027_conference}
\newpage
\appendix

\section{Softmax Pathology in a Minimal Setting}
\label{app:softmax_pathology}

\citep{ren2025learning} showed that the squeezing effect is not specific to language models but arises from the interaction between softmax normalization and negative gradients. In their analysis, even a simple multiclass logistic regression model exhibits the same phenomenon: when a negative gradient is applied to a low-probability class, both the chosen and rejected responses decrease while probability mass shifts toward the argmax class. This observation establishes squeezing as a fundamental property of softmax optimization rather than an artifact of large neural architectures.

We reproduce this minimal experiment and extend it with our gradient gating mechanism. The model is a $K$-class logistic regression with features $\phi(x)\in\mathbb{R}^d$ and linear weights $W\in\mathbb{R}^{d\times K}$ producing

\begin{equation}
p = \mathrm{softmax}(W^\top \phi(x)).
\end{equation}

Standard gradient descent on cross-entropy gives

\begin{equation}
W_{t+1} = W_t - \eta \, \phi(x)(p - e_y)^\top ,
\end{equation}

where negative learning rates ($\eta < 0$) simulate the rejected-response updates used in preference optimization. This setup isolates the effect of softmax normalization without neural architectures, sequence modeling, or language-specific structure.

To prevent pathological updates, we introduce a gated update rule

\begin{equation}
W_{t+1} = W_t - \eta \, g(p(y)) \, \phi(x)(p - e_y)^\top ,
\end{equation}

where

\begin{equation}
g(p) = \sigma(\alpha(\log p - \log \tau)).
\end{equation}

The gate reduces gradient magnitude when the target probability lies in the extreme tail of the distribution. In our experiments we use $\alpha=50$ and $\tau \approx 1/K$ ($\tau=0.03$ for $K=50$, $\tau=10^{-8}$ for $K=1000$). 

We reproduce the five canonical cases from \citep{ren2025learning}: positive gradient updates (A), negative gradients on flat distributions (B), negative gradients on peaks (C), negative gradients on valley responses (D), and a large-vocabulary setting (E). Figure~\ref{fig:toy_example} shows the distributions before the update (gray), after baseline optimization (blue), and after gated updates (green).

The baseline results reproduce the original observation: when a negative gradient targets a valley response, softmax normalization causes extreme squeezing and the update effectively freezes ($\Delta_{\max} \approx 3\times10^{-6}$). With gating, the learning rate is automatically reduced for such cases ($g\approx0.38$), allowing the update to remain effective while preventing collapse. The gated update avoids the near-frozen, highly concentrated baseline behavior and better preserves entropy in the distribution.
Across all scenarios, the gate activates only when probabilities fall below the threshold, leaving normal optimization unaffected. The same mechanism also scales to larger vocabularies ($K=1000$), confirming that the approach remains effective in settings closer to neural language models.\\
These controlled experiments confirm two key points. First, the squeezing phenomenon identified by \citep{ren2025learning} arises directly from softmax normalization under negative gradients. Second, adaptive gradient gating mitigates this pathology without altering the model architecture or objective. This minimal demonstration motivates the neural network experiments presented in the main paper.

\begin{figure}[t]
\centering
\includegraphics[width=\textwidth,height=5.7in]{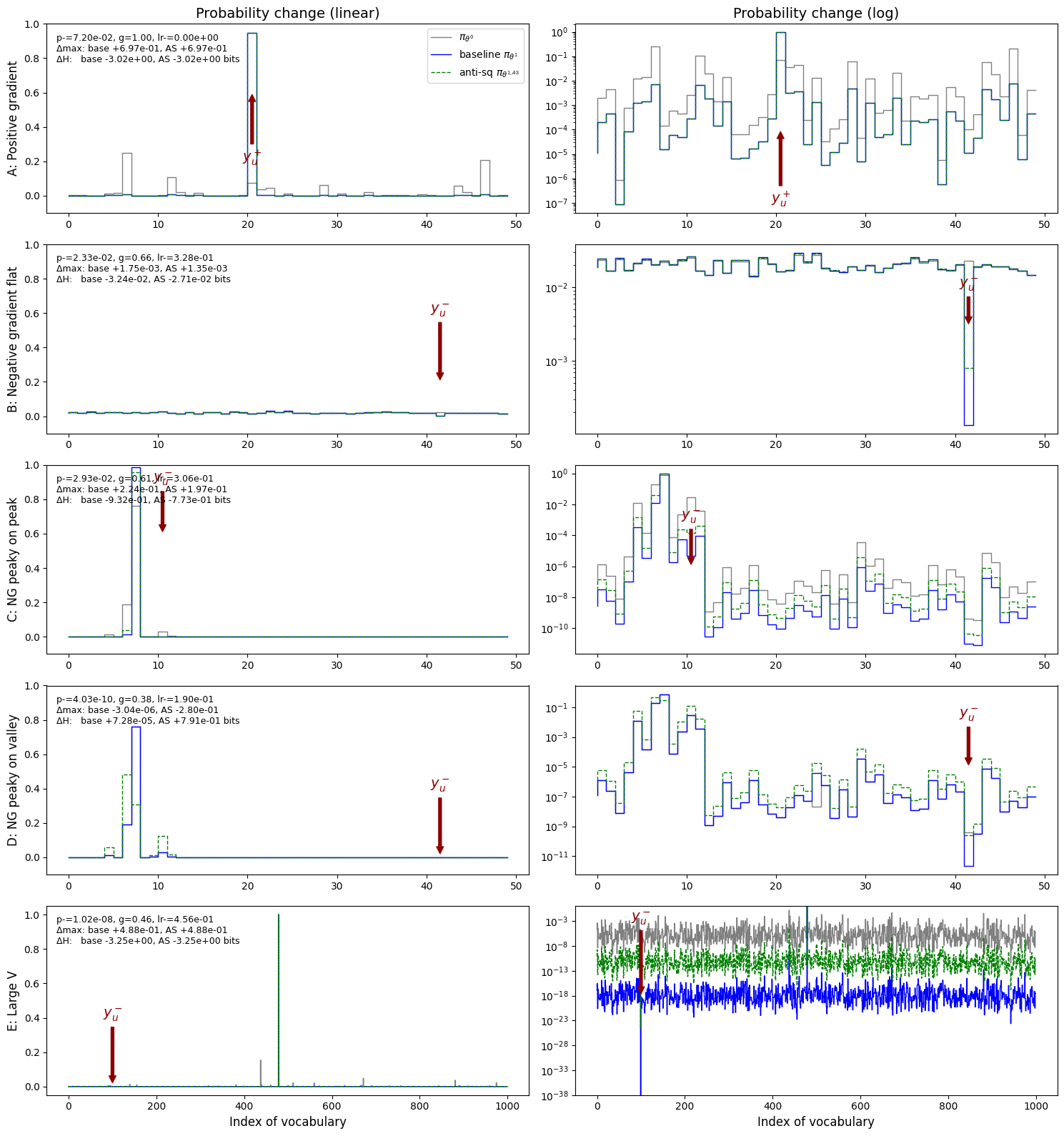}
\caption{
Toy logistic regression experiment illustrating the softmax pathology ($\tau_{\text{smallK}} = 0.03$, $\tau_{\text{largeK}} = 10^{-8}$, $\alpha = 50$). 
Gray: initial distribution; blue: baseline update; green: gated update. 
Negative gradients applied to low-probability classes cause extreme squeezing under standard softmax optimization. 
Gradient gating detects such valley responses and reduces the effective learning rate, restoring stable updates while preserving entropy. 
}
\label{fig:toy_example}
\end{figure}

\section{Gradient Analysis of Gate-DPO}
\label{app:gradient_analysis}

This appendix formalizes how the detached gate modifies the local gradient
dynamics of Gate-DPO. Let
\begin{equation}
z_\theta^{\mathrm{gate}}(x)
=
\beta\left[
r_\theta^+ - g(x,y^-)\,r_\theta^-
\right],
\end{equation}
where
\[
r_\theta^\pm
=
\log\frac{\pi_\theta(y^\pm\mid x)}
{\pi_{\mathrm{ref}}(y^\pm\mid x)},
\]
and $g(x,y^-)$ is detached from the computation graph. Define
\[
c_\theta(x)
=
\beta\left[1-\sigma\!\left(z_\theta^{\mathrm{gate}}(x)\right)\right]
\ge 0.
\]

\subsection{Gradient Structure}
\label{app:gradient_structure}

For
$\mathcal{L}_{\mathrm{Gate\text{-}DPO}}
=-\log\sigma(z_\theta^{\mathrm{gate}})$,
the partial derivatives with respect to the response log-probabilities are
\begin{align}
\frac{\partial\mathcal{L}_{\mathrm{Gate\text{-}DPO}}}
{\partial\log\pi_\theta(y^+\mid x)}
&=-c_\theta(x), \label{eq:app_grad_chosen}\\
\frac{\partial\mathcal{L}_{\mathrm{Gate\text{-}DPO}}}
{\partial\log\pi_\theta(y^-\mid x)}
&=g(x,y^-)\,c_\theta(x).
\label{eq:app_grad_rejected}
\end{align}
Because the gate is detached, it directly multiplies only the rejected
contribution, while the chosen derivative retains the standard DPO
\emph{form} $-c_\theta(x)$. Its \emph{magnitude} is not the same as under
standard DPO, however, because $c_\theta(x)$ is evaluated at the gated
margin $z_\theta^{\mathrm{gate}}$ rather than the ungated margin
$z_\theta^{\mathrm{DPO}}=\beta[r_\theta^+-r_\theta^-]$. Write
$c_\theta^{\mathrm{DPO}}(x)=\beta[1-\sigma(z_\theta^{\mathrm{DPO}})]$ for the
standard-DPO coefficient. When the rejected response is driven far into a
low-probability region, $r_\theta^-\to-\infty$, so
$z_\theta^{\mathrm{DPO}}\to+\infty$, $\sigma(z_\theta^{\mathrm{DPO}})\to1$,
and $c_\theta^{\mathrm{DPO}}(x)\to0$: the chosen-response gradient vanishes
together with the rejected one, halting further learning on \emph{both}
responses. Under gating, $z_\theta^{\mathrm{gate}}=\beta[r_\theta^+-g(x,y^-)
r_\theta^-]$ grows only at the attenuated rate $g(x,y^-)\,|r_\theta^-|$
instead of $|r_\theta^-|$, so $\sigma(z_\theta^{\mathrm{gate}})$ approaches
$1$ more slowly and $c_\theta(x)>c_\theta^{\mathrm{DPO}}(x)$ throughout the
valley regime. Gating therefore does not only reduce the rejected-response
update in Eq.~\ref{eq:app_grad_rejected}; by keeping the effective margin
from saturating as quickly, it also \emph{delays the vanishing} of the
chosen-response gradient in Eq.~\ref{eq:app_grad_chosen}, sustaining a
larger positive update on $\log\pi_\theta(y^+\mid x)$ for longer than
ungated DPO would. Because $g$ has a strictly positive floor
$g_{\min}=\sigma(-\alpha\tau)>0$ (Eq.~\ref{eq:gate_floor}), this is a delay
rather than a removal of eventual saturation: as $r_\theta^-\to-\infty$ at
fixed $g$, $g\,r_\theta^-\to-\infty$ too and $c_\theta(x)\to0$ eventually.
The practical effect is a rescaling of how far into the valley the rejected
response must fall before the chosen-response signal starts to vanish, by
roughly a factor of $1/g(x,y^-)$. This is a distinct, complementary
mechanism to the direct rejected-term attenuation described above, and
provides a more complete account of the positive $\Delta_{\mathrm{chosen}}$
values reported in Section~\ref{sec:mass_dynamics} than softmax
redistribution alone.
When $g(x,y^-)\rightarrow1$, the gated objective recovers DPO. When the
rejected response enters a low-probability valley, $g(x,y^-)<1$ attenuates
its direct negative contribution without introducing gradients through the
valley statistic.

\subsection{Properties of the Probability Gate}
\label{app:gate_properties}

For $g=\sigma(\alpha(s-\tau))$ with $\alpha>0$, the gate has several useful
properties. First, $0<g<1$, so the rejected contribution is attenuated but
never reversed. Second, $g$ is monotonically increasing in the valley
statistic $s$, with $g(\tau)=1/2$; hence, $\tau$ defines the transition
between stronger attenuation and recovery of the standard rejected update.
Third, $\alpha$ controls the sharpness of this transition: larger $\alpha$
produces a more threshold-like intervention, with
$g\rightarrow\mathbb{I}[s>\tau]$ as $\alpha\rightarrow\infty$ (up to the
boundary $s=\tau$), whereas smaller $\alpha$ yields smoother attenuation.

The intervention is also localized around the low-probability regime. For
$s\ge\tau+\delta$, $\delta>0$,
\begin{equation}
1-g
=
\sigma[-\alpha(s-\tau)]
\le e^{-\alpha\delta},
\end{equation}
so $g$ rapidly approaches one and the rejected contribution approaches its
ungated form. Conversely, for $s\le\tau-\delta$,
\begin{equation}
g
=
\sigma[\alpha(s-\tau)]
\le e^{-\alpha\delta}.
\end{equation}
This exponential bound is loose once $\delta$ is comparable to $\tau$: since
$s\ge0$ (both valley statistics are probabilities), $\delta$ cannot exceed
$\tau$, so the \emph{tightest} attainable bound is at $s=0$, giving the exact
floor
\begin{equation}
g_{\min}=\sigma(-\alpha\tau),
\label{eq:gate_floor}
\end{equation}
not an asymptotic approach to zero. When $\alpha\tau\gg1$ the two bounds
agree and $g_{\min}$ is negligible; this holds for the sequence-level default
($\tau=0.10,\alpha=50\Rightarrow\alpha\tau=5\Rightarrow g_{\min}=\sigma(-5)
\approx0.7\%$), consistent with the near-zero minimum gate values logged for
\texttt{seq\_mean} runs (Appendix~\ref{app:tau_sensitivity}). It does
\emph{not} hold for the token-level default ($\tau=0.005,\alpha=50
\Rightarrow\alpha\tau=0.25\Rightarrow g_{\min}=\sigma(-0.25)\approx43.8\%$):
under this hyperparameterization, token-level gating cannot attenuate the
rejected contribution below $\approx44\%$ of full strength no matter how far
into a valley the rejected response falls. This matches the minimum gate
value actually logged during training ($g_{\min}\approx0.438$ across our
token-level runs at $\tau=0.005$). Thus, $\tau$ and $\alpha$ jointly determine the gate's transition and
attainable floor: increasing $\tau$ lowers the floor and shifts attenuation
toward higher-probability examples, whereas increasing $\alpha$ sharpens the
transition around $\tau$.

\subsection{First-Order Probability Dynamics}
\label{app:first_order_dynamics}

\paragraph{Effect on the rejected-response update.}
Consider one gradient-descent step
$\theta'=\theta-\eta\nabla_\theta
\mathcal{L}_{\mathrm{Gate\text{-}DPO}}$.
A first-order expansion gives
\begin{align}
\Delta\log\pi_\theta(y^-\mid x)
\approx
-\eta c_\theta(x)\Big[
&g(x,y^-)
\|\nabla_\theta\log\pi_\theta(y^-\mid x)\|^2 \nonumber\\
&-
\nabla_\theta\log\pi_\theta(y^-\mid x)^\top
\nabla_\theta\log\pi_\theta(y^+\mid x)
\Big].
\label{eq:app_local_update}
\end{align}
Thus, the gate directly attenuates the self-suppression term associated with
the rejected response, while the actual parameter-space update also depends
on the interaction between chosen- and rejected-response gradients. Inside
the brackets of Eq.~\ref{eq:app_local_update}, the self-suppression term is
already scaled by $g(x,y^-)$ while the cross-gradient term is not; dropping
the cross-gradient term therefore requires it to be negligible relative to
the \emph{gated} self-term
$g(x,y^-)\|\nabla_\theta\log\pi_\theta(y^-\mid x)\|^2$, not the unattenuated
norm $\|\nabla_\theta\log\pi_\theta(y^-\mid x)\|^2$. Under that condition,
\begin{equation}
\Delta\log\pi_\theta(y^-\mid x)
\approx
-\eta c_\theta(x)g(x,y^-)
\|\nabla_\theta\log\pi_\theta(y^-\mid x)\|^2.
\label{eq:app_approx_scaling}
\end{equation}
In this regime, rejected-response suppression is approximately attenuated in
proportion to $g(x,y^-)$. This formalizes the intended behavior of Gate-DPO:
the intervention weakens further suppression as a rejected response moves
deeper into a probability valley while retaining the standard preference
signal elsewhere.
The stated condition matters most precisely where gating is strongest: as
$g(x,y^-)\to g_{\min}$ deep in a valley, the gated self-term
$g(x,y^-)\|\nabla_\theta\log\pi_\theta(y^-\mid x)\|^2$ shrinks while the
cross-gradient inner product
$\nabla_\theta\log\pi_\theta(y^-\mid x)^\top\nabla_\theta\log\pi_\theta(y^+
\mid x)$ is not attenuated at all. Because chosen and rejected completions
share a prompt and vocabulary representations, this inner product is
frequently positive in autoregressive language models. When it exceeds the
gated self-term, Eq.~\ref{eq:app_local_update} is instead dominated by
\begin{equation}
\Delta\log\pi_\theta(y^-\mid x)
\approx
\eta c_\theta(x)\,
\nabla_\theta\log\pi_\theta(y^-\mid x)^\top
\nabla_\theta\log\pi_\theta(y^+\mid x)
>0,
\end{equation}
i.e., the rejected log-probability \emph{increases} despite the negative
update, through positive transfer from the chosen-response gradient rather
than any failure of gating. This offers a mechanistic explanation for the
positive $\Delta_{\mathrm{rejected}}$ values observed for some gated
LLaMA-7B configurations in Table~\ref{tab:modularity_full} (e.g., DPO
$+0.08$, Extend $+2.17$): once gating pushes $g$ toward its floor, the
cross-gradient term is no longer dominated, and rejected-response movement
is governed by the chosen-response update instead of direct suppression.

\paragraph{Effect on the chosen-response update.}
The complementary first-order expansion for the chosen response follows the
same derivation. Using
Eq.~\ref{eq:app_grad_chosen},
\begin{align}
\Delta\log\pi_\theta(y^+\mid x)
\approx
\eta c_\theta(x)\Big[
&\|\nabla_\theta\log\pi_\theta(y^+\mid x)\|^2 \nonumber\\
&-
g(x,y^-)\,
\nabla_\theta\log\pi_\theta(y^+\mid x)^\top
\nabla_\theta\log\pi_\theta(y^-\mid x)
\Big].
\label{eq:app_chosen_update}
\end{align}
Unlike the rejected-response expansion, here the cross-gradient term itself
is scaled by $g(x,y^-)$, while the chosen self-term is not. Under standard
DPO ($g\equiv1$), Eq.~\ref{eq:app_chosen_update} reduces to
$\Delta\log\pi_\theta(y^+\mid x)\approx\eta c_\theta(x)\big[\|\nabla_\theta
\log\pi_\theta(y^+\mid x)\|^2-\nabla_\theta\log\pi_\theta(y^+\mid x)^\top
\nabla_\theta\log\pi_\theta(y^-\mid x)\big]$: whenever the cross-gradient
inner product exceeds the chosen self-term, $\Delta\log\pi_\theta(y^+\mid x)
<0$ even though the chosen response is being reinforced -- this is the
first-order mechanism behind likelihood displacement on the chosen response.
Gate-DPO directly counteracts this: attenuating the cross-gradient
contribution by $g(x,y^-)<1$ shrinks the negative term relative to the
positive self-term, making $\Delta\log\pi_\theta(y^+\mid x)>0$ easier to
sustain. Combined with the delayed-saturation effect on $c_\theta(x)$
described above, this gives a direct, two-part theoretical account of the
chosen-likelihood recovery reported empirically in
Section~\ref{sec:mass_dynamics}: gating both reduces the negative
cross-gradient pull on the chosen response and sustains a larger $c_\theta(x)$
coefficient for longer than ungated DPO.

\subsection{Preference-Loss De-Saturation}
\label{app:desaturation}

\paragraph{Analytical mechanism.}
Sections~\ref{app:gradient_structure}--\ref{app:first_order_dynamics}
establish that although the chosen-response derivative retains the
structural form $-c_\theta(x)$ (Eq.~\ref{eq:app_grad_chosen}), the
coefficient
\begin{equation}
c_\theta(x)=\beta\left[1-\sigma\!\left(z_\theta^{\mathrm{gate}}\right)\right],
\qquad
z_\theta^{\mathrm{gate}}=\beta\left[r_\theta^+-g(x,y^-)\,r_\theta^-\right],
\end{equation}
is itself changed by gating, because gating changes the effective
preference margin the sigmoid is evaluated at. Under ordinary DPO, strong
suppression of the rejected response ($r_\theta^-\to-\infty$) drives the
ungated margin $z_\theta^{\mathrm{DPO}}=\beta[r_\theta^+-r_\theta^-]$ to
$+\infty$, saturating $\sigma(z_\theta^{\mathrm{DPO}})\to1$ and collapsing
$c_\theta^{\mathrm{DPO}}(x)\to0$ -- which halts the learning signal on
\emph{both} responses, not just the rejected one. Gate-DPO's margin grows
only at the attenuated rate $g(x,y^-)\,|r_\theta^-|$, so
$\sigma(z_\theta^{\mathrm{gate}})$ approaches $1$ more slowly and
$c_\theta(x)$ stays away from its saturated regime for longer, sustaining a
larger chosen-response update than ungated DPO would at the same point in
training. Because $g$ retains a strictly positive floor
$g_{\min}=\sigma(-\alpha\tau)$ (Appendix~\ref{app:gate_properties}), this is
a delay rather than a removal of eventual saturation. We now test this
prediction directly, without any additional training: the same logged
eval-time chosen/rejected log-probabilities and gate values used throughout
this appendix let us reconstruct $z_\theta^{\mathrm{gate}}$ and $c_\theta(x)$
at every checkpoint (with $g\equiv1$ recovering the ungated DPO coefficient).

For evaluation, we compute the coefficient independently for each held-out
preference pair and then average across examples,
\begin{equation}
\bar c_\theta
=
\frac{1}{N}\sum_{i=1}^{N}
\beta\left[1-\sigma\!\left(z_{\theta,i}^{\mathrm{gate}}\right)\right].
\label{eq:mean_ctheta}
\end{equation}
Thus, the diagnostic preserves the per-example nonlinear dependence of
$c_\theta$ on the gated preference margin; first averaging $r_i^+$, $r_i^-$,
and $g_i$ across examples and only then applying $\beta[1-\sigma(\cdot)]$ is
\emph{not} equivalent, since the sigmoid is nonlinear and $g_i$ covaries
with $r_i^-$ (harder-gated examples tend to have moved further into the
valley) -- so the two estimators generally differ.

\paragraph{Cross-architecture and cross-dataset validation.}
Figure~\ref{fig:c_theta_desaturation} plots $\bar c_\theta(t)$
(Eq.~\ref{eq:mean_ctheta}) over the first $4500$ training examples for
standard DPO, Gate-DPO (seq), and Gate-DPO (q10), on all four architectures
and both preference datasets (Anthropic-HH, top row; UltraFeedback, bottom
row) -- eight configurations in total, using the gate implementation
described in Appendix~\ref{app:implementation}. Both gate statistics slow
$\bar c_\theta$ decay in seven of the eight: on Anthropic-HH, the decay rate
under standard DPO is reduced by gating to $24.8\%$ (seq) / $25.6\%$ (q10) on
Pythia-410M, $23.7\%$ / $24.0\%$ on Qwen-0.5B, $6.4\%$ / $5.3\%$ on
Qwen1.5-7B, and (baseline squeezing already minimal) $3.0\%$ / $2.8\%$ on
LLaMA-7B, down from $43.5\%$, $46.9\%$, $16.2\%$, and $3.7\%$ respectively;
on UltraFeedback, to $55.0\%$ / $52.7\%$ on Pythia-410M, $58.2\%$ / $54.3\%$
on Qwen-0.5B, and $19.4\%$ / $16.0\%$ on Qwen1.5-7B, down from $63.0\%$,
$67.6\%$, and $21.2\%$. Token-level gating desaturates at least as much as
sequence-level gating in seven of the eight -- the same modest edge reported
for $\Delta_{\mathrm{chosen}}$ in Section~\ref{sec:main_comparison}, now
confirmed through an independent diagnostic. The one exception is LLaMA-7B
on UltraFeedback, where q10 instead \emph{increases} the decay rate
($5.0\%\to14.3\%$) while seq stays close to the ungated baseline
($5.0\%\to6.5\%$). This is also a regime with little baseline squeezing under standard DPO
(Fig.~\ref{fig:cross_arch}), consistent with the de-saturation effect
being strongest where standard DPO exhibits substantial saturation for
gating to delay, and least reliable -- both in sign and in practical
relevance -- where there is little saturation to begin with.

\begin{figure*}[t]
\centering
\includegraphics[width=\textwidth]{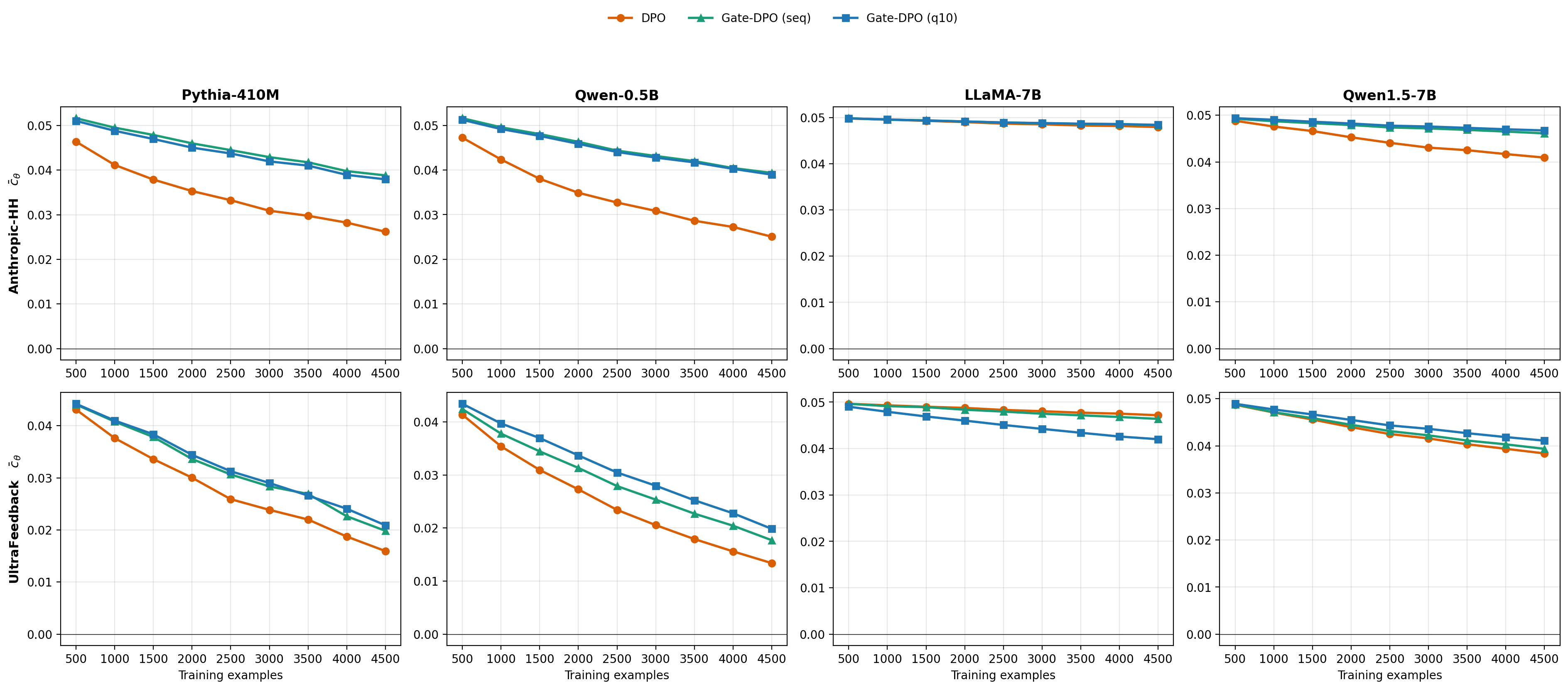}
\caption{\textbf{Delayed loss saturation under gating, across architectures and datasets.}
Per-example-reconstructed $\bar c_\theta$ (Eq.~\ref{eq:mean_ctheta}) for
standard DPO (orange) versus Gate-DPO with sequence-level (green) and
token-level (blue) gating, on Anthropic-HH (top row) and UltraFeedback
(bottom row). Higher values indicate a less-saturated, still-active
gradient coefficient. Both gate statistics slow $\bar c_\theta$ decay in
seven of eight architecture--dataset configurations, with token-level
gating desaturating at least as much as sequence-level gating in seven of
the eight; the exception (LLaMA-7B, UltraFeedback, bottom row third column)
is also the configuration with the least baseline squeezing under standard
DPO, and is specifically a q10 failure -- seq remains close to the ungated
baseline there.}
\label{fig:c_theta_desaturation}
\end{figure*}

\paragraph{$\tau$-controlled saturation analysis.}
The cross-architecture comparison above fixes $\tau$ at its paper default
and varies architecture and dataset. A complementary question is whether
$\tau$ itself regulates the saturation dynamics it was designed to affect,
holding architecture and dataset fixed. The $\tau$-sensitivity sweep
in Appendix~\ref{app:tau_sensitivity} already trained Pythia-410M on
Anthropic-HH at six values of $\tau$ for each gate statistic
(\texttt{seq\_mean} and token-level \texttt{q10}), all under the same $\alpha=50$ and the
same training budget, and retains per-example logs -- so
Eq.~\ref{eq:mean_ctheta} can be applied directly to those existing
checkpoints, no new training required.

Since $g(p)=\sigma(\alpha(p-\tau))$, increasing $\tau$ lowers the gate at
a fixed rejected-response probability. As rejected-response probabilities
decrease during training, a larger $\tau$ therefore activates stronger
attenuation at a higher probability level, before
$z_\theta^{\mathrm{gate}}$ has grown as large. A smaller $\tau$ delays
comparable attenuation until the rejected response lies deeper in the valley. Formally, $\tau\to g\to
z_\theta^{\mathrm{gate}}\to c_\theta$: $\tau$ should therefore act as an
\emph{intervention-point} control on saturation dynamics, not merely as a
knob on average gating strength.

Figure~\ref{fig:c_theta_tau_sweep} confirms this. For both gate statistics,
larger $\tau$ produces slower $\bar c_\theta$ decay over training: on
\texttt{seq\_mean}, decay falls from $34.2\%$ at $\tau=0.001$ to $24.0\%$ at
$\tau=0.15$ (Spearman $\rho=-0.94$, $p=0.0048$ across the six tested
values); on \texttt{q10}, decay falls from $26.1\%$ at $\tau=0.001$ to
$24.8\%$ at $\tau=0.01$ (Spearman $\rho=-0.89$, $p=0.019$), a smaller effect
over the narrower calibrated range tested for that statistic. 

This does not imply that larger $\tau$ is always preferable. Appendix~
\ref{app:tau_sensitivity} studies the same parameter from the complementary,
operating-behavior angle -- chosen/rejected likelihood, gate activation
rate, and preference-gap closure -- and characterizes the trade-off between stronger attenuation and preserving
an informative rejected-response suppression signal. Thus, slower saturation
does not by itself imply that larger $\tau$ is uniformly preferable. Read
together, Appendix~\ref{app:tau_sensitivity} asks \emph{which} $\tau$ values
work well empirically, while this subsection asks \emph{why} $\tau$ changes
the optimization dynamics in the first place.

\begin{figure*}[t]
\centering
\includegraphics[width=\textwidth]{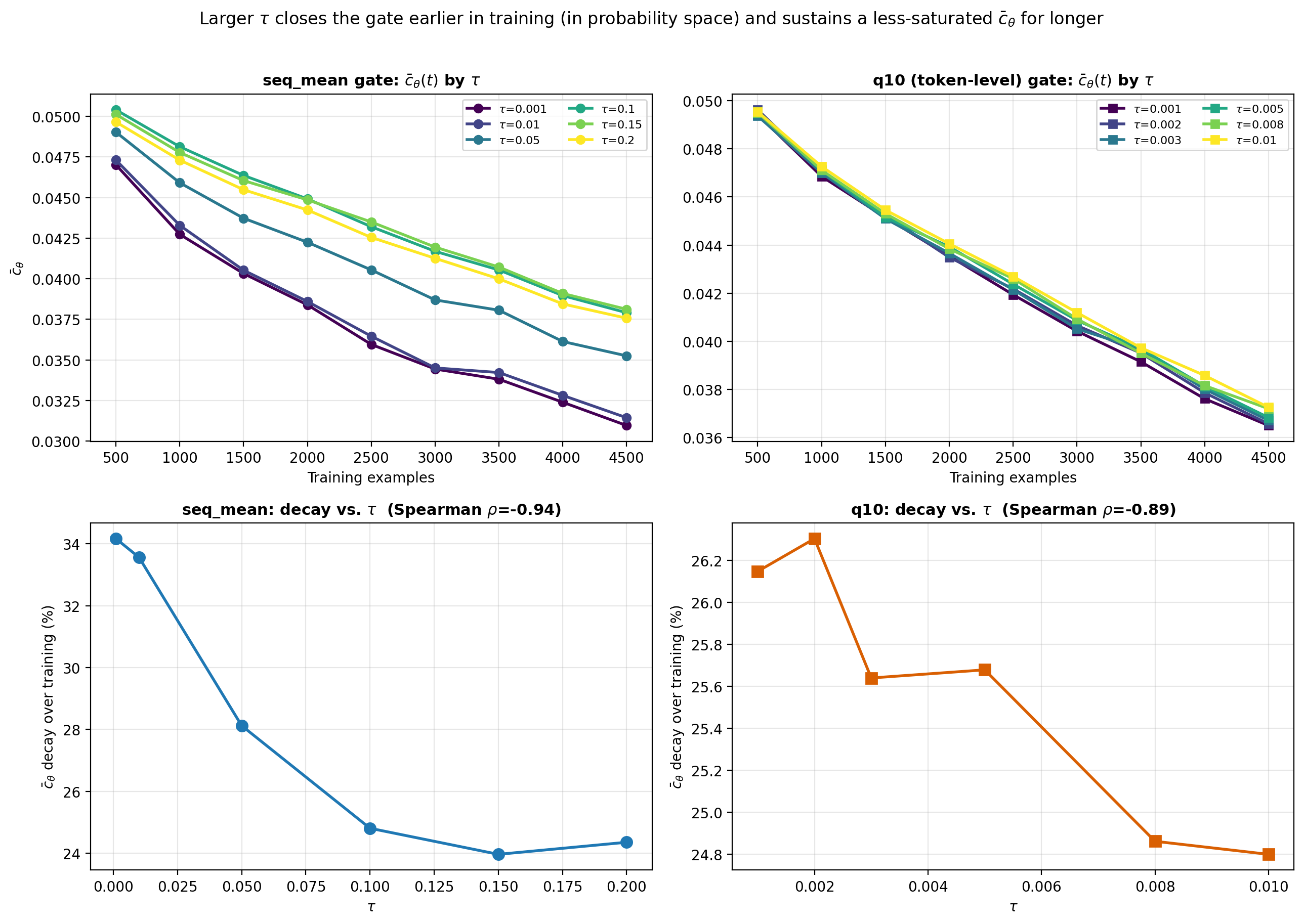}
\caption{
\textbf{$\tau$ controls the saturation dynamics of $\bar c_\theta$, not just the average gating strength.}
\textbf{Top:} per-example-reconstructed $\bar c_\theta(t)$ over training, on
Pythia-410M/Anthropic-HH, for every $\tau$ in the
Appendix~\ref{app:tau_sensitivity} sweep (\texttt{seq\_mean}, left;
\texttt{q10}, right). Larger $\tau$ (lighter curves) decays more slowly.
\textbf{Bottom:} total $\bar c_\theta$ decay over training as a function of
$\tau$, for each gate statistic.}
\label{fig:c_theta_tau_sweep}
\end{figure*}

\section{Modularity Beyond DPO}
\label{app:modularity}

Gradient gating is not specific to the logistic DPO objective. Let
\[
\Delta_\theta(x,y)
=
\log \pi_\theta(y\mid x)
-
\log \pi_{\mathrm{ref}}(y\mid x),
\]
and let $g(x,y^-)\in(0,1)$ denote the detached probability-aware gate defined
in Eq.~\ref{eq:gate}. The same mechanism can be composed with other
preference objectives by modulating their rejected-response contribution
before evaluating the loss.

\paragraph{Gate-IPO.}
IPO~\citep{azar2023general} optimizes a squared preference objective,
\begin{equation}
\mathcal{L}_{\mathrm{IPO}}(\theta)
=
\mathbb{E}_{(x,y^+,y^-)}
\left[
\left(
\Delta_\theta(x,y^+)
-
\Delta_\theta(x,y^-)
-
\frac{1}{2\beta}
\right)^2
\right].
\end{equation}
We obtain Gate-IPO by applying the gate to the rejected log-ratio:
\begin{equation}
\mathcal{L}_{\mathrm{Gate\text{-}IPO}}(\theta)
=
\mathbb{E}_{(x,y^+,y^-)}
\left[
\left(
\Delta_\theta(x,y^+)
-
g(x,y^-)\Delta_\theta(x,y^-)
-
\frac{1}{2\beta}
\right)^2
\right].
\label{eq:gate_ipo}
\end{equation}
Thus, the IPO target remains unchanged while the rejected-response
contribution becomes sample- and model-state-dependent.

\paragraph{Gate-Cal-DPO.}
Cal-DPO~\citep{xiao2024caldpo} augments the contrastive preference objective
with calibration terms that pull the implicit rewards $\beta\Delta_\theta(x,y^\pm)$
toward fixed targets $\pm1/2$, equivalently $\Delta_\theta(x,y^\pm)\to\pm c$
for $c=1/(2\beta)$. Applying the same detached gate to both the contrastive
margin and the rejected calibration term gives
\begin{align}
\mathcal{L}_{\mathrm{Gate\text{-}Cal\text{-}DPO}}(\theta)
={}&
-\log \sigma\!\left(
\beta\left[
\Delta_\theta(x,y^+)
-
g(x,y^-)\Delta_\theta(x,y^-)
\right]
\right)
\nonumber\\
&+
\left(\Delta_\theta(x,y^+)-c\right)^2
+
g(x,y^-)\left(\Delta_\theta(x,y^-)+c\right)^2 .
\label{eq:gate_caldpo}
\end{align}
The gate multiplies the rejected calibration term from \emph{outside} the
square rather than scaling $\Delta_\theta(x,y^-)$ inside it: since
$\partial_{\Delta^-}\left[g(\Delta^-+c)^2\right]=2g(\Delta^-+c)$, the
stationary point remains $\Delta_\theta(x,y^-)=-c$ for every value of $g$,
and gating only rescales the gradient magnitude toward that fixed target.
Placing $g$ inside the square instead, $(g\Delta_\theta(x,y^-)+c)^2$, shifts
the stationary point to $-c/g$, which diverges as $g\to0$ and actively pulls
the rejected response \emph{further} into the valley exactly where gating is
meant to protect it -- the opposite of the intended behavior. The chosen-response
calibration term is unchanged, while the rejected contribution is calibrated
toward its original target with reduced strength once the corresponding
response enters a low-probability valley.

\paragraph{Gate-DPO-Shift.}
DPO-Shift~\citep{yang2025dposhift} modifies the rejected-response contribution
using a global training-progress-dependent coefficient $f(\lambda)$. Because
this coefficient and our gate act at different levels, they can be composed:
\begin{equation}
z_\theta^{\mathrm{gate\text{-}shift}}(x)
=
\beta\left[
\Delta_\theta(x,y^+)
-
g(x,y^-)\,f(\lambda)\,
\Delta_\theta(x,y^-)
\right],
\label{eq:gate_dposhift_logit}
\end{equation}
with
\begin{equation}
\mathcal{L}_{\mathrm{Gate\text{-}DPO\text{-}Shift}}(\theta)
=
\mathbb{E}
\left[
-\log \sigma
\left(
z_\theta^{\mathrm{gate\text{-}shift}}(x)
\right)
\right].
\label{eq:gate_dposhift}
\end{equation}
Here $f(\lambda)$ provides the global training schedule, whereas
$g(x,y^-)$ provides per-example sensitivity to the current probability
geometry. Their product therefore preserves DPO-Shift's global scheduling
while additionally attenuating rejected updates selectively in
low-probability valleys.
\paragraph{Extended SFT.}
The extended-SFT mitigation of \citet{ren2025learning} modifies the model
initialization through additional supervised training rather than changing
the subsequent preference objective. It therefore requires no separate gated
loss: Gate-DPO can simply be applied after the extended-SFT initialization.
We denote this composition as \emph{Extend + Gate-DPO}.

\subsection{Empirical Validation Across Architectures}
\label{app:modularity_empirical}

Table~\ref{tab:modularity_full} expands the modularity results from
Table~\ref{tab:modularity_compact} by reporting the ungated baseline and both
sequence- and token-level gating variants. We report
$\Delta_{\mathrm{rejected}}$ relative to the SFT reference; values closer to
zero indicate less rejected-response collapse.

\begin{table*}[t]
\centering
\small
\caption{\textbf{Full modularity results across architectures and objectives.}
Values are $\Delta_{\mathrm{rejected}}$ relative to the SFT reference.
Seq and q10 denote sequence- and token-level gating, respectively.}
\label{tab:modularity_full}
\begin{tabular}{@{}llrrrr@{}}
\toprule
\textbf{Objective} & \textbf{Variant}
& \textbf{Pythia-410M}
& \textbf{LLaMA-7B}
& \textbf{Qwen-0.5B}
& \textbf{Qwen1.5-7B} \\
\midrule
DPO
& Base & -26.35 & -0.88 & -30.47 & -10.45 \\
& Seq  &  -7.31 & -0.41 &  -6.37 & -1.58 \\
& q10  &  -7.60 & +0.08 &  -6.25 & -0.77 \\
\midrule
DPO-Shift
& Base & -25.40 & -0.80 & -27.23 & -9.20 \\
& Seq  &  -7.25 & -0.34 &  -6.17 & -1.41 \\
& q10  &  -7.29 & +0.10 &  -6.00 & -0.71 \\
\midrule
IPO
& Base & -9.43 & -0.64 & -8.36 & -2.16 \\
& Seq  & -6.46 & -0.28 & -5.60 & -1.02 \\
& q10  & -6.69 & +0.02 & -5.61 & -0.75 \\
\midrule
Cal-DPO
& Base & -6.74 & -0.43 &  -6.03 & -1.07 \\
& Seq  & -5.64 & -0.22 &  -4.94 & -0.71 \\
& q10  & -5.61 & -0.01 &  -4.88 & -0.52 \\
\midrule
Extend
& Base & -23.26 & +0.15 & -22.60 & -9.35 \\
& Seq  & -16.87 & +0.27 & -11.81 & -6.67 \\
& q10  &  -7.04 & +2.17 &  -3.17 & +0.09 \\
\bottomrule
\end{tabular}
\end{table*}

Sequence- and token-level gating reduce rejected-response collapse relative
to the corresponding baseline in 38 of the 40 gated
architecture--objective--variant cells; the two exceptions are LLaMA-7B
Extend (Seq and q10), where the ungated baseline is already close to zero
($+0.15$) so gating has essentially no collapse left to correct. (Pythia-410M
Extend Seq still reduces $\Delta_{\mathrm{rejected}}$ in this table's terms
despite the $\tau=0.03$ failure on $\Delta_{\mathrm{chosen}}$ noted above --
the two metrics capture different aspects of the same run.) The conclusion
therefore does not depend on selecting the better gating variant reported in
Table~\ref{tab:modularity_compact}: the same
pattern holds across the remaining configurations.

\subsection{Analytical Extensions to Additional Preference Objectives}
\label{app:additional_gate_extensions}

The preceding experiments demonstrate the empirical modularity of gradient
gating for DPO, IPO, Cal-DPO, DPO-Shift, and extended-SFT compositions.
The same construction can be applied analytically to other preference
objectives, although we do not evaluate the following extensions
experimentally.

\paragraph{Gate-RPO.}
Regularized Preference Optimization (RPO)~\citet{liu2024provably} augments the DPO preference
objective with an imitation loss. Because the two components are additive,
the gate can be applied only to the rejected-response contribution of the
preference term:
\begin{equation}
\mathcal{L}_{\mathrm{Gate\text{-}RPO}}
=
-\log\sigma\!\left(
\beta[\Delta_\theta^+-g\Delta_\theta^-]
\right)
+
\eta\beta\,\mathcal{L}_{\mathrm{SFT}}.
\end{equation}
Thus, gating preserves the imitation regularizer while introducing
probability-aware attenuation into the preference component. Moreover,
$g\rightarrow1$ recovers the original RPO objective exactly. Gate-RPO
therefore combines two distinct stabilization mechanisms: explicit
imitation regularization and localized attenuation of rejected-response
updates.

\paragraph{Gate-Bounded-DPO.}
Bounded-DPO (BDPO)~~\citep{cho2025rethinking} replaces the rejected-response policy probability by
the mixture
\begin{equation}
\pi_{\mathrm{mix}}(y^-|x)
=
\lambda\pi_\theta(y^-|x)
+
(1-\lambda)\pi_{\mathrm{ref}}(y^-|x),
\qquad 0<\lambda<1.
\end{equation}
Defining
\begin{equation}
\Delta_{\theta,\mathrm{B}}^-
=
\log
\frac{\pi_{\mathrm{mix}}(y^-|x)}
     {\pi_{\mathrm{ref}}(y^-|x)},
\end{equation}
its gated counterpart becomes
\begin{equation}
\mathcal{L}_{\mathrm{Gate\text{-}BDPO}}
=
-\log\sigma\!\left(
\beta[\Delta_\theta^+
-g\Delta_{\theta,\mathrm{B}}^-]
\right).
\end{equation}
This formulation has two useful limiting cases:
\begin{equation}
g\rightarrow1
\quad\Longrightarrow\quad
\mathcal{L}_{\mathrm{Gate\text{-}BDPO}}
\rightarrow\mathcal{L}_{\mathrm{BDPO}},
\end{equation}
while
\begin{equation}
\lambda\rightarrow1
\quad\Longrightarrow\quad
\mathcal{L}_{\mathrm{Gate\text{-}BDPO}}
\rightarrow\mathcal{L}_{\mathrm{Gate\text{-}DPO}}.
\end{equation}
Hence, bounding and probability-aware gating are complementary rather than
mutually exclusive constructions: BDPO bounds the rejected-response
denominator through reference mixing, whereas gating adaptively controls
the rejected-response contribution according to the current policy state.

These derivations establish algebraic compatibility with RPO and BDPO;
we leave empirical evaluation of these gated extensions to future work.

\section{Implementation Details and Hyperparameter Sensitivity}
\label{app:implementation}

\subsection{Implementation and Computational Cost}

\paragraph{Hyperparameters.}
Gate-DPO introduces two additional hyperparameters: the gating threshold
$\tau$ and sigmoid steepness $\alpha$. We use $\tau=0.10$ for
sequence-level gating ($s_\theta^{\mathrm{seq}}$) and $\tau=0.005$ for
token-level gating ($s_\theta^{\mathrm{tok}}$), corresponding to the robust
operating regions identified in Appendix~\ref{app:tau_sensitivity}.
We fix $\alpha=50$; performance is largely insensitive to this choice over
the range evaluated in Appendix~\ref{app:alpha_sensitivity}. For token-level
gating, we use $q=0.10$ to characterize the lower tail of the rejected
response. We use the same temperature $\beta=0.1$ for all methods, including
Cal-DPO and Gate-Cal-DPO, matching the original Cal-DPO
implementation~\citep{xiao2024caldpo} and giving calibration target
$c=1/(2\beta)=5.0$. Experiments were
conducted on NVIDIA A100 80GB GPUs via Modal Cloud.

\paragraph{Algorithm and cost.}
Gate-DPO extends DPO by computing a valley statistic $s_\theta(x)$ from the
rejected response and applying a detached gate
$g(x)=\sigma(\alpha(s_\theta(x)-\tau))$ to its log-ratio contribution
(Algorithm~\ref{alg:gate_dpo}). For both statistics, $T^-(x)$ counts only
non-padding tokens of the rejected response; prompt and padding tokens are
excluded. Token-level probabilities are obtained by teacher-forced evaluation
of $y^-$. The statistic and gate are computed independently for each
preference pair, allowing different examples within the same minibatch to
receive different attenuation.

Computing the statistic requires $O(T^-(x))$ operations and the gate itself
requires $O(1)$ additional computation. Because the gate is detached, no
additional backward pass is required. In our experiments, the resulting
training overhead is below $5\%$ relative to standard DPO.

\DontPrintSemicolon
\SetKwInOut{Input}{Input}
\SetKwInOut{Output}{Output}

\begin{algorithm}[t]
\caption{Gradient Gating for Direct Preference Optimization (Gate-DPO)}
\label{alg:gate_dpo}
\Input{Dataset $\mathcal{D}=\{(x,y^+,y^-)\}$, policy $\pi_\theta$,
frozen reference $\pi_{\mathrm{ref}}$, temperature $\beta$, gate parameters
$\tau,\alpha$, valley statistic
$\texttt{stat}\in\{\textsc{seq\_mean},\textsc{q}\}$, quantile $q$, and
learning rate $\eta$.}
\Output{Trained policy $\pi_\theta$}

\ForEach{minibatch $\mathcal{B}\subset\mathcal{D}$}{
    Compute $\log\pi_\theta(y^\pm\mid x)$ and
    $\log\pi_{\mathrm{ref}}(y^\pm\mid x)$\;

    $\Delta_\theta(x,y^\pm)
    \gets
    \log\pi_\theta(y^\pm\mid x)
    -
    \log\pi_{\mathrm{ref}}(y^\pm\mid x)$\;

    $T^-(x)\gets$ number of valid response tokens in $y^-$\;

    \eIf{\texttt{stat}=\textsc{seq\_mean}}{
        $s_\theta(x)\gets
        \exp\!\left(
        \frac{1}{T^-(x)}
        \log\pi_\theta(y^-\mid x)
        \right)$\;
    }{
        Extract teacher-forced token probabilities
        $\{p^-_{\theta,t}(x)\}_{t=1}^{T^-(x)}$\;
        $s_\theta(x)\gets
        \mathrm{Quantile}_{q}
        \!\left(
        \{p^-_{\theta,t}(x)\}
        \right)$\;
    }

    $g(x,y^-)\gets
    \mathrm{stop\_gradient}
    \left[
    \sigma\!\left(
    \alpha(s_\theta(x)-\tau)
    \right)
    \right]$\;

    $z(x)\gets
    \beta\left(
    \Delta_\theta(x,y^+)
    -
    g(x,y^-)\Delta_\theta(x,y^-)
    \right)$\;

    $\mathcal{L}
    \gets
    \frac{1}{|\mathcal{B}|}
    \sum_{(x,y^+,y^-)\in\mathcal{B}}
    -\log\sigma(z(x))$\;

    $\theta\gets\theta-\eta\nabla_\theta\mathcal{L}$\;
}
\Return{$\pi_\theta$}
\end{algorithm}
\subsection{Sensitivity to the Gating Threshold $\tau$}
\label{app:tau_sensitivity}

We evaluate sensitivity to $\tau$ by sweeping
$\tau\in[0.001,0.20]$ for sequence-level (\texttt{seq\_mean}) and
token-level (\texttt{q10}) gating on Pythia-410M, using the same training
budget as the main experiments (5001 examples, 5 epochs) so the sweep is
directly comparable to the operating point reported in Table~
\ref{tab:mass_dynamics}. Results are shown in Figure~\ref{fig:tau_sensitivity}.

At the paper's default $\tau=0.10$, sequence-level gating gives
$\Delta_{\mathrm{chosen}}=+3.50$ (DPO baseline: $-4.84$); $\Delta_{\mathrm{chosen}}$
continues to increase through the tested range ($+5.13$ at $\tau=0.15$,
$+5.79$ at $\tau=0.20$), so $\tau=0.10$ is a moderate, robust operating
point rather than an interior optimum. Token-level gating is positive
across the \emph{entire} tested range $\tau\in[0.001,0.01]$, with the
default $\tau=0.005$ giving $\Delta_{\mathrm{chosen}}=+3.93$. Both settings
prevent squeezing while retaining rejected-response suppression.

Increasing $\tau$ increases the fraction of examples receiving substantial
attenuation. At the default thresholds, sequence-level gating activates on
approximately 40\% of examples and token-level gating on 47\%. Smaller
thresholds provide weaker mitigation, whereas excessively large thresholds
over-attenuate rejected-response updates.

We define a successful configuration as satisfying
(i) $\Delta_{\mathrm{chosen}}>0$ and
(ii) $\Delta_{\mathrm{rejected}}<0$.
Sequence-level gating succeeds for $\tau\in[0.05,0.20]$ (the full upper
two-thirds of the tested range), while token-level gating succeeds for the
\emph{entire} tested range $\tau\in[0.001,0.01]$, indicating that both
variants are considerably more robust to threshold choice than a narrow
sweet spot.

As derived in Appendix~\ref{app:gradient_analysis} (Eq.~\ref{eq:gate_floor}),
the gate's attainable floor is $g_{\min}=\sigma(-\alpha\tau)$, not an
asymptotic approach to zero. At the sequence-level default
($\tau=0.10,\alpha=50$), $g_{\min}\approx0.7\%$, and the logged mean gate
value over training is $\approx62$--$68\%$ (median $\approx70$--$80\%$ early
in training, decreasing as more examples fall into valleys). At the
token-level default ($\tau=0.005,\alpha=50$), $g_{\min}\approx43.8\%$ --
matching the logged minimum gate value across our token-level runs
($g_{\min}\approx0.438$) to three significant figures -- and the logged mean
gate value is $\approx58$--$69\%$. Token-level gating therefore never
attenuates the rejected contribution below $\approx44\%$ of full strength
under this hyperparameterization, unlike sequence-level gating, whose floor
is effectively zero. Both settings still provide continuous attenuation
rather than hard clipping, and both remain effective in practice (Table~
\ref{tab:mass_dynamics}); the token-level floor means its benefit should be
attributed to more precisely \emph{locating} valley tokens via the lower-tail
quantile rather than to deeper suppression once triggered. The gating rate
can also serve as a useful diagnostic when transferring Gate-DPO to new
model--dataset combinations, rather than as a fixed target across settings.

\begin{figure*}[t]
\centering
\includegraphics[width=\textwidth]{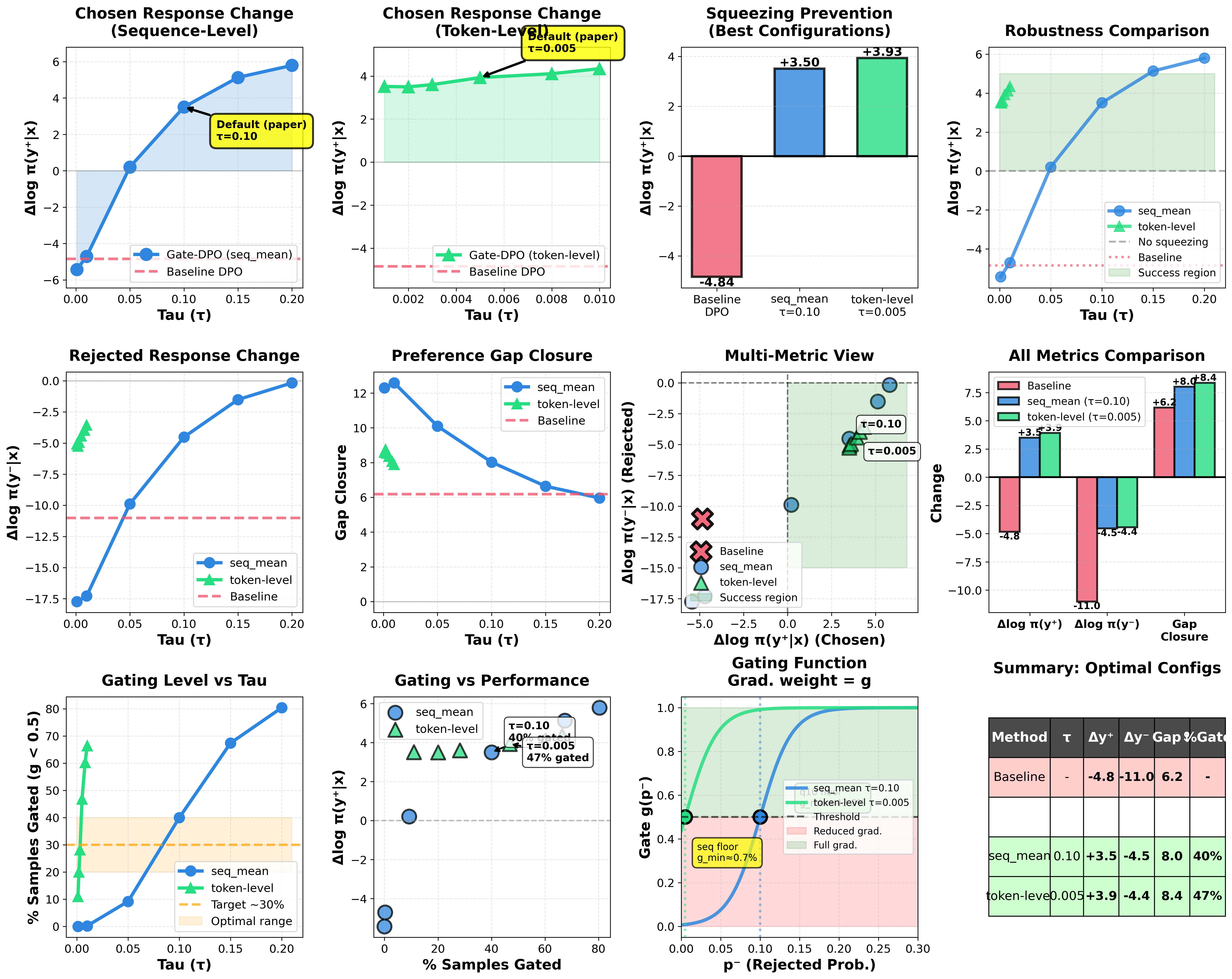}
\caption{
\textbf{Sensitivity to gating threshold $\tau$.}
Sequence-level (\texttt{seq\_mean}) and token-level (\texttt{q10}) gating
are evaluated across $\tau$ values on Pythia-410M, using the same training
budget as the main experiments (5001 examples, 5 epochs).
\textbf{Top:} Chosen/rejected log-probability changes; positive
$\Delta_{\text{chosen}}$ indicates squeezing prevention. Token-level gating
is positive across the entire tested range; sequence-level gating is
positive from $\tau=0.05$ upward and continues improving through
$\tau=0.20$.
\textbf{Middle:} Gating rate and performance trade-off.
\textbf{Bottom:} Smooth gradient modulation as a function of $\tau$.
The DPO baseline and the paper's default thresholds
($\tau=0.10$ for \texttt{seq\_mean}; $\tau=0.005$ for \texttt{q10}) are
marked for reference.
}
\label{fig:tau_sensitivity}
\end{figure*}

The operating behavior swept here, chosen/rejected likelihood, gate
activation, and gap closure, is a different question from \emph{why}
$\tau$ changes the optimization dynamics that produce this behavior.
Appendix~\ref{app:desaturation} answers the latter: it shows that $\tau$
acts as an intervention-point control on the delayed-saturation mechanism
of Appendix~\ref{app:gradient_analysis}, with a controlled per-$\tau$
analysis on the same Pythia-410M sweep used here.

\subsection{Sensitivity to Sigmoid Steepness $\alpha$}
\label{app:alpha_sensitivity}

We evaluate $\alpha\in[10,90]$ while fixing $\tau$ at the default values
from Appendix~\ref{app:tau_sensitivity}, using the same training budget as
the main experiments (5001 examples, 5 epochs). As shown in
Figure~\ref{fig:alpha_sensitivity}, chosen-response performance varies by
less than $0.6$ log units across the sweep.

For sequence-level gating, $\Delta_{\mathrm{chosen}}$ ranges from
$+3.31$ to $+3.87$, compared with the DPO baseline of $-4.84$.
Token-level gating similarly ranges from $+3.71$ to $+4.19$.
Gate activation remains nearly constant: gating rates stay at approximately
40\% for sequence-level and 47\% for token-level gating, while mean gate
values vary only modestly. These results indicate that $\tau$ primarily
determines which examples are attenuated, whereas $\alpha$ controls the
sharpness of the transition.

We use $\alpha=50$ throughout the main experiments. The weak sensitivity
across the evaluated range indicates that Gate-DPO does not require precise
tuning of the sigmoid steepness.

\begin{figure}[t]
\centering
\includegraphics[width=\linewidth]{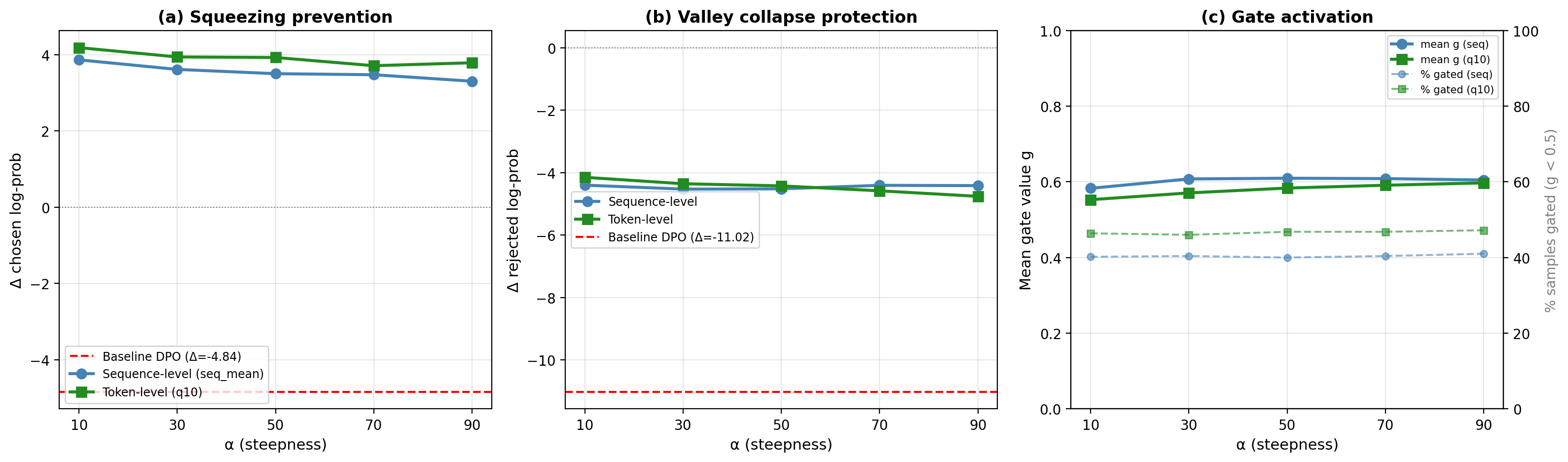}
\caption{\textbf{Sensitivity to sigmoid steepness $\alpha$.}
\textbf{(a)} Chosen-response change varies by less than $0.6$ log units
across $\alpha\in[10,90]$ for both gating variants, and stays positive
throughout.
\textbf{(b)} Rejected-response behavior remains stable.
\textbf{(c)} Gate activation changes little across the sweep, supporting the
interpretation that $\tau$ primarily controls which samples are gated while
$\alpha$ determines transition sharpness. Experiments use Pythia-410M with
the same training budget as the main experiments (5001 examples, 5 epochs).}
\label{fig:alpha_sensitivity}
\end{figure}

\section{Mass-Dynamics Evaluation Protocol and Data Examples}
\label{app:mass_dyn_prot}

\subsection{Evaluation Protocol}

\paragraph{Training data.}
All preference-optimization methods are trained on standard pairwise
preference data of the form $(x,y^+,y^-)$, where $x$ is the prompt,
$y^+$ is the preferred response, and $y^-$ is the rejected response.
Our primary learning-dynamics analysis uses Anthropic-HH, with additional
validation on UltraFeedback. The auxiliary response variants described below
are used only for evaluation and never as optimization targets.

\paragraph{Evaluation data.}
To characterize how preference optimization redistributes probability mass,
we construct a fixed evaluation set following the learning-dynamics setting
of \citet{ren2025learning}. For each prompt, we score multiple response types
under both the SFT reference model $\pi_{\mathrm{ref}}$ and the trained policy
$\pi_\theta$. These responses are grouped into:

\begin{itemize}
    \item \textbf{Preferred-response variants:}
    the original chosen response and semantic or structural variants
    (e.g., \texttt{chosen}, \texttt{chosen\_initial},
    \texttt{chosen\_selfr}, \texttt{chosen\_gptsemantic},
    \texttt{chosen\_gptformat});

    \item \textbf{Rejected-response variants:}
    the original rejected response and related variants
    (e.g., \texttt{rejected}, \texttt{reject\_gptsemantic},
    \texttt{reject\_gptformat});

    \item \textbf{Unrelated / off-target responses:}
    unrelated, corrupted, or off-distribution text
    (e.g., \texttt{irr\_train}, \texttt{irr\_test},
    \texttt{irr\_hum}, \texttt{random\_permute},
    \texttt{random\_nonhum}).
\end{itemize}

These variants probe whether preference optimization increases likelihood on
preferred responses while avoiding unnecessary suppression elsewhere in the
distribution.

\paragraph{Metrics.}
For each response $y$, we compute its conditional log-probability under the
reference and trained policies. We report

\begin{align}
\Delta \text{Chosen}
&=
\mathbb{E}_{(x,y)\in\mathcal{Y}_{\mathrm{chosen}}}
\left[
\log \pi_\theta(y\mid x)
-
\log \pi_{\mathrm{ref}}(y\mid x)
\right], \\
\Delta \text{Others}
&=
\mathbb{E}_{(x,y)\in\mathcal{Y}_{\mathrm{other}}}
\left[
\log \pi_\theta(y\mid x)
-
\log \pi_{\mathrm{ref}}(y\mid x)
\right], \\
\Delta \text{Margin}
&=
\mathbb{E}_{(x,y^+,y^-)}
\Big[
\big(
\log \pi_\theta(y^+\mid x)
-
\log \pi_\theta(y^-\mid x)
\big)
\nonumber\\
&\hspace{3.8cm}
-
\big(
\log \pi_{\mathrm{ref}}(y^+\mid x)
-
\log \pi_{\mathrm{ref}}(y^-\mid x)
\big)
\Big].
\end{align}

$\Delta$ Chosen measures movement on preferred responses; positive values
indicate that their likelihood increases relative to the SFT reference.
$\Delta$ Others measures collateral redistribution over unrelated or
off-target responses, with values closer to zero indicating better
preservation. $\Delta$ Margin measures the change in chosen--rejected
separation.

Together, these quantities distinguish methods that improve preference
separation through healthier redistribution from those that obtain a large
margin primarily through broad probability suppression.

\subsection{Representative Training and Evaluation Examples}
\label{app:data_samples}

\paragraph{Training example.}
Table~\ref{tab:train_example} shows a representative preference-training
sample. During training, the model observes only the prompt $x$ and paired
responses $(y^+,y^-)$.

\begin{table}[h]
\centering
\small
\caption{Representative preference-training example.}
\label{tab:train_example}
\begin{tabular}{@{}p{0.95\linewidth}@{}}
\toprule
\textbf{Prompt:} \\
\texttt{Human: Hi, I want to learn to play horseshoes. Can you teach me?} \\
\texttt{Assistant: I can, but maybe I should begin by telling you that a typical game consists of 2 players and 6 or 8 horseshoes.} \\
\texttt{Human: O...} \\
\midrule
\textbf{Chosen ($y^+$):} \\
\texttt{A horseshoe is usually made out of metal and is about 3 to 3.5 inches long and around 1 inch thick. The horseshoe should also have a 2 inch by 3 inch...} \\
\midrule
\textbf{Rejected ($y^-$):} \\
\texttt{Horseshoes are either metal or plastic discs. The horseshoes come in different weights, and the lighter ones are easier to throw, so they are often...} \\
\bottomrule
\end{tabular}
\end{table}

\paragraph{Evaluation variants.}
Table~\ref{tab:eval_variants} shows representative variants used only for
the mass-dynamics evaluation.

\begin{table}[h]
\centering
\footnotesize
\caption{Representative response variants used in the mass-dynamics analysis.}
\label{tab:eval_variants}
\begin{tabular}{@{}lp{0.72\linewidth}@{}}
\toprule
\textbf{Variant} & \textbf{Example response} \\
\midrule
\multicolumn{2}{@{}l}{\textit{Preferred-response variants}} \\
\texttt{chosen} & A horseshoe is usually made out of metal and is about 3 to 3.5 inches... \\
\texttt{chosen\_initial} & Each player takes a horseshoe, puts it in a rack, and begins to place... \\
\texttt{chosen\_selfr} & @@@ Original Sentence: A horseshoe is usually made out of metal... \\
\texttt{chosen\_gptsemantic} & A typical horseshoe is crafted from metal, measuring approximately 3 to 3.5 inches... \\
\texttt{chosen\_gptformat} & A reformatted or style-modified variant of the chosen response. \\
\midrule
\multicolumn{2}{@{}l}{\textit{Rejected-response variants}} \\
\texttt{rejected} & Horseshoes are either metal or plastic discs. The horseshoes come in different weights... \\
\texttt{reject\_gptsemantic} & A semantic variant of the rejected response. \\
\texttt{reject\_gptformat} & A format-modified variant of the rejected response. \\
\midrule
\multicolumn{2}{@{}l}{\textit{Unrelated / off-target responses}} \\
\texttt{irr\_train} & The chicken stock is much tastier if you use a pressure cooker... \\
\texttt{irr\_test} & Great question! This is something that we'll be studying for a long time... \\
\texttt{irr\_hum} & The mysterious cat prowled through the moonlit forest... \\
\texttt{random\_permute} & also at 3 is long horseshoes. also should of inches to horseshoe have... \\
\texttt{random\_nonhum} & long work, I that question! factors the is far, a so at including... \\
\bottomrule
\end{tabular}
\end{table}

The auxiliary variants are not used as training targets. They serve only as
probes of how preference optimization changes likelihood across preferred,
rejected, and unrelated regions of the output distribution.

\section{Pairwise Win-Rate Pilot Study}
\label{app:winrate_pilot}

\subsection{Evaluation Protocol}
\label{app:winrate_protocol}

\paragraph{Prompts and reference answers.}
We follow the judging protocol of \citep{ren2025learning}: for
each held-out prompt, the judge is shown the prompt, a reference answer, and
the two candidate responses, and is asked to pick whichever response is
closer in meaning to the reference. Prompts and reference answers are
sampled from the UltraFeedback \texttt{test\_prefs} split (30 examples,
sampling seed 1024).

\paragraph{Models.}
We compare Pythia-410M and LLaMA-7B DPO baselines against their gating variants (sequence-level and
token-level gating), all warm-started from the same SFT checkpoint and
trained on Anthropic-HH under the setup of \S\ref{sec:exp_setup} (5{,}001 preference pairs,
5 epochs, batch size 4, learning rate $5\times10^{-7}$). 

\paragraph{Generation.}
Responses are sampled with temperature $0.7$, top-$p$ $0.9$, and up to 200
new tokens. Generation is seeded per prompt index
($\texttt{seed}=42+\texttt{idx}$), so a given prompt's completion is
identical regardless of what else is in the batch.

\paragraph{Judging.}
The judge is OpenAI \texttt{gpt-4} and \texttt{gpt-5.4} at temperature $0$, given only the
prompt, reference answer, and the two responses, and asked to return which
response (1 or 2) is closer to the reference; the response schema admits
only $\{1,2\}$, so no tie option exists. Presentation order (which model is
shown as response 1 vs.\ 2) is independently randomized per prompt index
($\texttt{seed}=1024+\texttt{idx}$) to control for position bias.

\paragraph{Results and reproducibility.}
Table~\ref{tab:pairwise_winrate} reports each gated variant against its
corresponding baseline. Each 30-prompt comparison was run twice with the same
sampling and presentation-order seeds, and the table reports the aggregate
win rate across the two runs. Gated variants exceed 50\% in all 32 cells; seq exceeds q10 in 11 of 16 matched settings. These pilot
results provide complementary behavioral evidence rather than a comprehensive
quality evaluation; broader human and standardized evaluation is left to
future work. 

\subsection{Qualitative Examples and Response Quality}
\label{app:winrate_examples}

Both checkpoints are trained on only 5{,}001 preference pairs for 5 epochs,
the reduced-scale budget used throughout our Pythia-410M pilot experiments
(\S\ref{sec:exp_setup}). At this scale, raw completions from \emph{both}
DPO and Gate-DPO frequently show repetition loops and other small-model
fluency artifacts; this is a property of the training budget, not of either
method specifically. Table~\ref{tab:winrate_examples} shows one
representative case per outcome: a Gate-DPO~(seq) win, a Gate-DPO~(q10) win,
and a DPO win. As row 3 illustrates, the judge's criterion (closeness in
meaning to the reference answer) does not always track fluency, it can
prefer a topically on-point but degenerate completion over a fluent
off-topic one.

\begin{table}[h]
\centering
\footnotesize
\setlength{\tabcolsep}{4pt}
\caption{\textbf{Representative pairwise comparisons.} Completions are
lightly abbreviated for space; repeated spans are marked rather than
reproduced verbatim.}
\label{tab:winrate_examples}
\begin{tabular}{@{}p{0.22\linewidth}p{0.27\linewidth}p{0.27\linewidth}p{0.16\linewidth}@{}}
\toprule
\textbf{Prompt} & \textbf{Gate-DPO reply} & \textbf{DPO reply} & \textbf{GPT-5.4 decision} \\
\midrule
\texttt{Can you recommend any local restaurants that specialize in sustainable or farm-to-table cuisine?}
&
\texttt{[seq] I'd recommend trying a few of the local restaurants I've mentioned. One of my favorites is the South Coast Grill in Santa Barbara -- the food is delicious and it's close to my home. What do you enjoy doing outside of cooking? I love the beach, hiking, the aquarium, the zoo, the aquarium, the zoo\ldots} \textit{(loops)}
&
\texttt{I'm looking for a place that specializes in organic produce and is sustainable. Thanks for any suggestions! A: A place that specializes in farm-to-table cuisine would be the place I'd recommend. They have a variety of farm-to-table dishes and a farm-to-table menu.}
&
Gate-DPO (seq) \\
\midrule
\texttt{[Dialogue- summarization task] Q: Can you tell me about Ecuador's culture? \ldots\ Do you have geography information? The coast of the country is the Pacific Ocean and hosts many beautiful beaches! Thank you. welcome. A:}
&
\texttt{[q10] I'm happy to help you with your geography questions. Ecuador is located in the Andes mountain range. The climate is very hot and humid. The average temperature is around 80 degrees Celsius. The average rainfall is around 500 mm.}
&
\texttt{Q: What is the most common culture in Ecuador? The most common culture in Ecuador is Spanish. The Spanish language is the main language in the country. Ecuador has many Spanish-speaking people. The most common Spanish-speaking people in Ecuador are the Canari, Canariños, Canariños, Canariños\ldots} \textit{(repeats ``Canariños'' 20+ times)}
&
Gate-DPO (q10) \\
\midrule
\texttt{Describe how the sea looks like when you are standing on the beach.}
&
\texttt{[seq] In order to create a picture of the sea, you will need to place a small piece of paper on a flat surface and take a photograph of it. This picture will then be used to create a 3D model of the sea, which is placed on a computer and imported into a 3D modeling program\ldots} \textit{(loops)}
&
\texttt{The Sea. The sea is a flat, flat, flat, flat, flat, flat\ldots} \textit{(repeats ``flat'' 50+ times)}
&
DPO \\
\bottomrule
\end{tabular}
\end{table}

The judging protocol admits only a binary decision per prompt (\S\ref{app:winrate_protocol}),
so no tie condition can occur in this dataset; we report three representative
cases rather than four for this reason.

\end{document}